\documentclass[letterpaper,11pt]{article}
\usepackage{etex}
\usepackage[margin=1in]{geometry}

\usepackage{amssymb,amsfonts,amsmath,amstext,amsthm,mathrsfs}
\usepackage{graphics,latexsym,epsfig,psfrag,wrapfig,comment,paralist}
\usepackage{xspace}
\usepackage{subcaption,bm,multirow}
\usepackage{booktabs}
\usepackage{mathdots}
\usepackage{bbold}
\usepackage{centernot}
\usepackage{prettyref}
\usepackage{listings}
\usepackage{tikz}
\usepackage{pgfplots}
\usetikzlibrary{shapes}
\usetikzlibrary{positioning, arrows, matrix, calc, patterns, fit}
\usepackage{caption}
\usepackage{array}
\usepackage{mdwmath}
\usepackage{multirow}
\usepackage{mdwtab}
\usepackage{eqparbox}
\usepackage{multicol}
\usepackage{amsfonts}
\usepackage{tikz}
\usepackage{multirow,bigstrut,threeparttable}
\usepackage{amsthm}
\usepackage{array}
\usepackage{bbm}
\usepackage{epstopdf}
\usepackage{mdwmath}
\usepackage{mdwtab}
\usepackage{eqparbox}
\usepackage{tikz}
\usepackage{latexsym}
\usepackage{amssymb}
\usepackage{bm}
\usepackage{graphicx}
\usepackage{mathrsfs}
\usepackage{epsfig}
\usepackage{psfrag}
\usepackage{setspace}
\usepackage[CJKbookmarks=true,
            bookmarksnumbered=true,
            bookmarksopen=true,
            colorlinks=true,
            citecolor=blue,
            linkcolor=blue,
            anchorcolor=red,
            urlcolor=blue
            ]{hyperref}
\usepackage{natbib}
\setcitestyle{authoryear,round}
\newtheoremstyle{break}
  {\topsep}{\topsep}%
  {\itshape}{}%
  {\bfseries}{}%
  {\newline}{}%

\usepackage{pgfplots}

\usepackage{amsthm}

\usepackage{hyperref}
\hypersetup{
    colorlinks,
    linkcolor={blue},
    citecolor={blue}
}
\usepackage{cleveref}

\DeclareRobustCommand{\abbrevcrefs}{%
\Crefname{theorem}{Thm.}{Thms.}%
}

\DeclareRobustCommand{\Cshref}[1]{{\abbrevcrefs\Cref{#1}}}

\usepackage{comment}
\usepackage{bbm}
\usepackage{algorithm}
\usepackage[noend]{algpseudocode}
\usepackage{caption}
\usepackage{subcaption}
\algnewcommand\algorithmicforeach{\textbf{for all}}
\algdef{S}[FOR]{ForAllDo}[1]{\algorithmicforeach\ #1\ }

\usepackage{xcolor}
\usepackage{tikz}
\usetikzlibrary{arrows,automata,positioning}

\DeclareMathOperator*{\argmin}{arg\,min}

\usepackage{enumitem}

\usepackage{etoc}
\usepackage[toc,title]{appendix}
\usepackage{bm}

\usepackage[T1]{fontenc}    
\usepackage{url}            
\usepackage{booktabs}       
\usepackage{amsfonts}       
\usepackage{nicefrac}       
\usepackage{microtype}      
\usepackage{amsmath,amssymb}
\usepackage{color}
\usepackage{setspace}

\newtheorem{definition}{Definition}[section]
\newtheorem{corollary}{\textbf{Corollary}}[section]
\newtheorem{theorem}{Theorem}[section]
\newtheorem{lemma}[theorem]{Lemma}

\newtheorem*{insight*}{\textbf{Observation}}
\newtheorem{theoremi}[theorem]{Theorem (informal)}

\newtheorem*{proposition*}{\textbf{Proposition}}

\newtheorem*{lemmai*}{\textbf{Lemma (informal)}}
\newtheorem{remark}{\textbf{Remark}}[section]

\title{Toward the Fundamental Limits of Imitation Learning}
\author{Nived Rajaraman, Lin F. Yang, Jiantao Jiao, Kannan Ramachandran \thanks{Nived Rajaraman is with the Department of Electrical Engineering and Computer Sciences, University of California, Berkeley. Lin F. Yang is with the Electrical and Computer Engineering Department at the University of California, Los Angeles. Jiantao Jiao is with the Department of Electrical Engineering and Computer Sciences and the Department of Statistics, University of California, Berkeley. Kannan Ramchandran is with the Department of Electrical Engineering and Computer Sciences, University of California, Berkeley. Email: \{nived.rajaraman, jiantao, kannanr\}@berkeley.edu ; linyang@ee.ucla.edu.}}
\date{\today}

\begin{document}

\maketitle

\begin{abstract}
\noindent Imitation learning (IL) aims to mimic the behavior of an expert policy in a sequential decision-making problem given only demonstrations. In this paper, we focus on understanding the minimax statistical limits of IL in episodic Markov Decision Processes (MDPs). We first consider the setting where the learner is provided a dataset of $N$ expert trajectories ahead of time, and cannot interact with the MDP. Here, we show that the policy which mimics the expert whenever possible is in expectation $\lesssim \frac{|\mathcal{S}| H^2 \log (N)}{N}$ suboptimal compared to the value of the expert, even when the expert follows an arbitrary stochastic policy. Here $\mathcal{S}$ is the state space, and $H$ is the length of the episode. Furthermore, we establish a suboptimality lower bound of $\gtrsim |\mathcal{S}| H^2 / N$ which applies even if the expert is constrained to be deterministic, or if the learner is allowed to actively query the expert at visited states while interacting with the MDP for $N$ episodes. To our knowledge, this is the first algorithm with suboptimality having no dependence on the number of actions, under no additional assumptions. We then propose a novel algorithm based on minimum-distance functionals in the setting where the transition model is given and the expert is deterministic. The algorithm is suboptimal by $\lesssim \min \{ H \sqrt{|\mathcal{S}| / N} ,\ |\mathcal{S}| H^{3/2} / N \}$, showing that knowledge of transition improves the minimax rate by at least a $\sqrt{H}$ factor.
\end{abstract}

\tableofcontents

\section{Introduction}

Imitation learning or apprenticeship learning is the study of learning from demonstrations in a sequential decision-making framework in the absence of reward feedback. The imitation learning problem differs from the typical setting of reinforcement learning in that the learner no longer has access to reward feedback to learn a good policy. In contrast, the learner is given access to expert demonstrations, with the objective of learning a policy that performs comparably to the expert's with respect to the \emph{unobserved reward function}. This is motivated by the fact that the desired behavior in typical reinforcement learning problems is easy to specify in words, but hard to capture accurately through manually-designed rewards \cite{Abbeel-Ng-ILviaIRL}. Imitation learning has shown remarkable success in practice over the last decade - the work of \cite{Abbeel-Ng-helicopter} showed that using pilot demonstrations to learn the dynamics and infer rewards can significantly improve performance in autonomous helicopter flight. More recently, the approach of learning from demonstrations has shown to improve the state-of-the-art in numerous areas: autonomous driving \cite{DeepQlearning-demonstrations,IL4agile-driving}, robot control \cite{ARGALL2009469}, game AI \cite{NIPS2018_8025,Vinyals2019} and motion capture \cite{Merel2017LearningHB} among others.

Following the approach pioneered by \cite{Beygelzimer05}, several works \cite{Ross-AIstats10,Brantley2020Disagreement-Regularized} show that carrying out supervised learning (among other approaches) to learn a policy provides black box guarantees on the suboptimality of the learner, in effect ``reducing'' the IL problem to supervised learning.  In particular when the expert follows a deterministic policy, \cite{Ross-AIstats10} discuss the behavior cloning approach, which is a supervised learning method for imitation learning by minimizing the number of mistakes made by the learner compared to the expert under the empirical state distribution in the expert demonstrations. However, the authors conclude that supervised learning could lead to severe error compounding due to the ``covariate shift problem'': the actual performance of learner depends on its own state distribution, whereas training takes place with respect to the expert's state distribution.
Furthermore, it remains to see how the reduction approach fares when the expert follows a general stochastic policy. As we discuss later, it turns out that in this setting, the reduction analysis is loose and can be improved: we instead use a novel coupling based approach to provide near optimal guarantees.

Nevertheless, the aforementioned reduction approach is quite popular in studying IL and shows that it suffices to approximately solve an intermediate problem to give one directional bounds on the suboptimality of a learner. But it is unclear whether a difficulty in solving the intermediate problem implies an inherent difficulty in solving the original imitation learning problem. In this work, we cast the imitation learning problem in the statistical decision theory framework and ask,

\begin{center}
\textbf{What are the statistical limits of imitation learning?}
\end{center}

We investigate this question in a tabular, epsiodic MDP over state space $\mathcal{S}$, action space $\mathcal{A}$ and episode length $H$, one of the most basic settings to start with. The value $J(\pi)$ of a (possibly stochastic) policy $\pi$ is defined as the expected cumulative reward accrued over the duration of an episode,
\begin{align} \label{eq:policy-value}
    J (\pi) = \mathbb{E}_\pi \left[ \sum\nolimits_{t=1}^H \mathbf{r}_t (s_t,a_t) \right]
\end{align}
where $\mathbf{r}_t$ is the unknown reward function of the MDP at time $t$, and the expectation is computed with respect to the distribution over trajectories $\{ (s_1,a_1),\cdots,(s_H,a_H) \} $ induced by rolling out the policy $\pi$. Denoting the expert's policy by $\pi^*$ and the learner's policy $\widehat{\pi}$, the \emph{subopimality} of a learner, which is a random variable, is defined as the difference in value of the expert's and learner's policies: $J(\pi^*) - J(\widehat{\pi})$.

In the imitation learning framework, we emphasize that the learner \emph{does not observe rewards} while interacting with the MDP, but is given access to expert demonstrations to learn a good policy. The learner is said to ``interact'' with the MDP by submitting a state-action pair to an oracle and receiving the next state, with the reward being hidden. We study the problem in the following $3$ settings:

\begin{enumerate}
    \item[(a)] \textbf{No-interaction:} The learner is provided a dataset of $N$ trajectories drawn by independently rolling out the expert policy through the MDP. The learner is not otherwise allowed to interact with the MDP.
    \item[(b)] \textbf{Known-transition:} The only difference compared to the no-interaction setting is that the MDP state transition functions and initial state distribution are exactly known to the learner.
    \item[(c)] \textbf{Active:} The learner is not given a dataset of expert demonstrations in advance. However, the learner is allowed to actively query the expert using previous expert feedback while interacting with the MDP for $N$ episodes.
\end{enumerate}

\noindent We remark that the active setting gives the learner more power than the no-interaction setting: it can just follow the observed expert actions (same as no-interaction setting), or actively compute a new policy on the fly based on expert feedback, and then use the most up-to-date policy to interact with the MDP. We mention two popular approaches, \textsc{Dagger} \cite{journals/jmlr/RossGB11} and \textsc{AggraVaTe} \cite{Ross2014ReinforcementAI} proposed for IL in the active setting.

\subsection{Main results}

The reduction approach in \cite{Ross-AIstats10} shows that if the expert policy is deterministic, and the probability of error in guessing the expert's action at each state is $\epsilon$, then $J(\pi^*) - J(\widehat{\pi}) \lesssim \min\{H,\epsilon H^2\}$. The appearance of this $H^2$ factor is called \emph{error compounding}. This quadratic behavior can be intuitively understood as the inability to get back on track once the learner makes a mistake. Indeed, the probability of making the first error at time $t$ is $\epsilon(1-\epsilon)^{t-1}$, and if we get completely lost thereafter we incur loss $H-t+1$. The suboptimality is therefore $H \epsilon + (H{-}1) \epsilon (1{-}\epsilon) + \cdots + \epsilon (1{-}\epsilon)^{H{-}1} \lesssim \min \{ H, H^2 \epsilon \}$. This informal argument supports the reduction of imitation learning to the problem of minimizing the empirical probability of error, known as ``behavior cloning''. For a more rigorous analysis we refer the reader to \cite{Ross-AIstats10}. We also provide a slight generalization of this result in \Cref{subsec:reduction-generalized}.

Is this error compounding inevitable or is it just a consequence of the behavior cloning algorithm? Our first contribution shows that it is \emph{fundamental} to the imitation learning problem without additional assumptions: even if the learner operates in the active setting and the expert is deterministic, no algorithm can beat the $H^2$ barrier. The term \textit{instance} is used to refer to the underlying MDP and the expert's policy in the imitation learning problem.

\begin{theoremi}[Formal version: \Cref{theorem:LB.p1}] \label{theorem:informal-1}
In the active setting, for any learner $\widehat{\pi}$, there exists an instance such that the suboptimality of the learner is at least $|\mathcal{S}| H^2/N$ up to universal constants, i.e., $J(\pi^*) - \mathbb{E} [J(\widehat{\pi})] \gtrsim |\mathcal{S}| H^2/N$. This lower bound applies even if the expert's policy $\pi^*$ is constrained to be deterministic.
\end{theoremi}

The key intuition behind this result is to identify that at states which were never visited during the learner's interactions with the MDP, the learner has no prior knowledge about the expert's policy. Furthermore, at such states the learner also has no knowledge about what state transitions are induced under different actions. With no available information, the learner is essentially forced to play an arbitrary policy on these states. A careful construction of the underlying MDP ultimately forces the learner to incur compounding errors when such states are visited, resulting in the lower bound.

Our next result shows that if the expert is deterministic, one can in fact achieve the bound $|\mathcal{S}|H^2/N$ in the no-interaction setting by the behavior cloning algorithm: 

\begin{theoremi}[Formal version: \Cref{theorem:det:Nsim0:UB.p2}] \label{theorem:informal-2}
When the expert's policy is deterministic, in the no-interaction setting, the expected suboptimality of a learner $\widehat{\pi}$ that carries out behavior cloning satisfies $J(\pi^*) - \mathbb{E} [J(\widehat{\pi})] \lesssim |\mathcal{S}| H^2 / N$ on any instance.
\end{theoremi}

We prove this result exactly as stated: by first bounding the population $0$-$1$ risk of the policy to be $|\mathcal{S}|/N$ and subsequently invoking the black box reduction \cite[Theorem 2.1]{Ross-AIstats10} to get the final bound on the expected suboptimality of the learner.

The optimality of behavior cloning and \Cref{theorem:informal-1} point to an interesting observation: the ability to actively query the expert \textit{does not} improve the minimax expected suboptimality beyond the no-interaction setting. An important implication of this result is that \textsc{Dagger} \cite{Ross-AIstats10} and other algorithms that necessitate an expert that can be actively queried, \emph{cannot improve over behavior cloning} in the worst case.

The closest relative to our bounds on behavior cloning in the deterministic expert setting is \cite{wen-sun-ILFO} the authors of which propose the \textsc{FAIL} algorithm. When the expert's policy is deterministic, in \cite[Theorem 3.3]{wen-sun-ILFO}, choosing $\Pi$ to be the set of all deterministic policies (of size $|\mathcal{A}|^{|\mathcal{S}|}$), shows that \textsc{FAIL} is suboptimal by $\lesssim \sqrt{|\mathcal{S}| |\mathcal{A}| H^5/ N}$ (ignoring logarithmic factors). In contrast, in \Cref{theorem:det:Nsim0:UB.p2}, we show that behavior cloning is suboptimal by $\lesssim |\mathcal{S}|H^2/N$ which always improves on the guarantee of \textsc{FAIL}: not only is it independent of $|\mathcal{A}|$, but has optimal dependence on $H$ and $N$. However, it is important to point out that the two results apply in slightly different settings and are not directly comparable: \textsc{FAIL} applies in the \textsc{ILfO} setting where the learner does not observe the actions played by the expert in the set of demonstrations, and only observes the visited states. In particular, the \textsc{ILfO} setting assumes that the reward function of the MDP only depends on the state visited and does not depend on the action chosen.

Prompted by the success of behavior cloning in the deterministic expert setting, it is natural to ask whether supervised learning reduction continues to be a good approach when the expert is stochastic. The reduction from \cite{Ross-AIstats10} indeed still guarantees that any policy with total variation (TV) distance $\epsilon$ with expert's action distribution has suboptimality $\lesssim H^2 \epsilon$ (see \Cref{lemma:nondet-reduction}). However there is a problem with invoking such a reduction to bound the suboptimality of the learner: the empirical action distribution at each state converges very slowly to the population expert action distribution under TV distance. Seeing as it corresponds to matching the expert's and learner's policies at different states which are distributions over $\mathcal{A}$, the population risk suffers from a convergence rate dependent on the number of actions, with rate $\sqrt{|\mathcal{A}|/N}$ instead of $N^{-1}$. In order to prove tight guarantees on the expected suboptimality of a policy, we are therefore forced to circumvent the reduction framework.

\begin{table}
  \centering
  \begin{tabular}{llcc}
    \toprule
    Expert & Setting & Upper bound & Lower bound \\
    \midrule
    Det. & No-interaction & $\frac{|\mathcal{S}| H^2}{N}$ (\Cshref{theorem:det:Nsim0:UB.p2}) & $\frac{|\mathcal{S}| H^2}{N}$ (\Cshref{theorem:LB.p1}) \\
    & Active & $\frac{|\mathcal{S}| H^2}{N}$ (\Cshref{theorem:det:Nsim0:UB.p2}) & $\frac{|\mathcal{S}| H^2}{N}$ (\Cshref{theorem:LB.p1}) \\
    & Known-transition & $\min \left\{ \frac{|\mathcal{S}|H^{3/2}}{N} , \sqrt{\frac{|\mathcal{S}|H^2}{N}} \right\}$ (\Cshref{theorem:det:NsimInf:UB.p1}) & $\frac{|\mathcal{S}| H}{N}$ (\Cshref{theorem:LB.p3}) \\
    Non. Det. & No-interaction & $\frac{|\mathcal{S}|H^2 \log (N)}{N}$ (\Cshref{theorem:nondet:UB}) & $\frac{|\mathcal{S}| H^2}{N}$ (\Cshref{theorem:LB.p1}) \\
    & Active & $\frac{|\mathcal{S}|H^2 \log (N)}{N}$ (\Cshref{theorem:nondet:UB}) & $\frac{|\mathcal{S}| H^2}{N}$ (\Cshref{theorem:LB.p1}) \\
    & Known-transition & $\frac{|\mathcal{S}|H^2 \log (N)}{N}$ (\Cshref{theorem:nondet:UB}) & $\frac{|\mathcal{S}| H}{N}$ (\Cshref{theorem:LB.p3}) \\
    \bottomrule
  \end{tabular}\\[10pt]
  \caption{Minimax expected suboptimality under different settings (all bounds are up to universal constants)}
  \label{table:results}
\end{table}

We analyze the \textsc{Mimic-Emp} policy in this setting, which carries out empirical risk minimization under log loss. This is a natural extension of empirical risk minimization under $0$-$1$ loss to the stochastic expert setting. The namesake for this policy follows from the fact that minimizing the empirical risk under log loss precisely translates to the learner playing the empirical expert policy distribution at states observed in the expert dataset. It is interesting to note that when the expert is determinstic, \textsc{Mimic-Emp} indeed still minimizes the empirical $0$-$1$ risk to $0$ and continues to be optimal in this setting.
We show that when the expert is stochastic, the expected suboptimality of \textsc{Mimic-Emp} does not depend on the number of actions. Moreover from the lower bound in \Cref{theorem:informal-1}, it is in fact minimax optimal up to logarithmic factors.

\begin{theoremi}[Formal version: \Cref{theorem:nondet:UB}]
In the no-interaction setting, the expected suboptimality of a learner $\widehat{\pi}$ carrying out \textsc{Mimic-Emp} is upper bounded by $J(\pi^*) - \mathbb{E} [J(\widehat{\pi})] \lesssim \frac{|\mathcal{S}| H^2 \log (N)}{N}$ on any instance. This result applies even when the expert plays a stochastic policy.
\end{theoremi}

The main ingredient in the proof of this result is a coupling argument which shows that the expected suboptimality of the learner results only from trajectories where the learner visits states unobserved in the expert dataset, and carefully bounding the probability of this event.

We next discuss the setting where the learner is not only provided expert demonstrations, but the state transitions functions of the MDP. The ``known-transition'' model appears frequently in robotics applications \cite{zhu2018reinforcement}, capturing the scenario where the learner has access to accurate models / simulators representing the dynamics of the system, but the rewards of the experts are difficult to summarize. Our key contribution here is to propose the \textsc{Mimic-MD} algorithm which breaks the lower bound in
\Cref{theorem:informal-1} and suppresses the issue of error compounding which the covariate shift problem entails. Recent works \cite{Brantley2020Disagreement-Regularized} propose algorithms that claim to bypass the covariate shift problem. However to the best of our knowledge, this is the first result that provably does so \emph{in the general tabular MDP setting without additional assumptions}.

\begin{theoremi}[Formal version: \Cref{theorem:det:NsimInf:UB.p1}]
In the known-transition setting, if the expert is deterministic, the expected suboptimality of a learner $\widehat{\pi}$ playing \textsc{Mimic-MD} is bounded by $J(\pi^*) - \mathbb{E} [J(\widehat{\pi})] \lesssim \min \{ H\sqrt{|\mathcal{S}|/N} ,\ |\mathcal{S}| H^{3/2}/N \}$.
\end{theoremi}

The novel element of \textsc{Mimic-MD} is a hybrid approach which mimics the expert on some states, and uses a minimum distance (MD) functional \cite{yatracos1985,donoho1988} to learn a policy on the remaining states. The minimum distance functional approach was recently considered in \cite{wen-sun-ILFO}, proposing to sequentially learn a policy by approximately minimizing a notion of discrepancy between the learner's state distribution and the expert's empirical state distribution. We remark that our approach is fundamentally different from matching the state distributions under the expert's and learner's policy: it crucially relies on exactly mimicking the expert actions on states visited in the dataset, and only applying the MD functional on the remaining states.
Interpreting the error of $\lesssim \min \{ H\sqrt{|\mathcal{S}|/N} ,\ |\mathcal{S}| H^{3/2}/N \}$ incurred by \textsc{Mimic-MD}, we make two observations: 
\begin{enumerate}
    \item[(i)] \textsc{Mimic-MD} improves the quadratic dependence on $H$ of the error incurred by behavior cloning (\Cref{theorem:informal-2}) by at least a $\sqrt{H}$ factor while preserving the dependence of the error on $|\mathcal{S}|$ and $N$;
    \item[(ii)] The error bound of $\lesssim H \sqrt{|\mathcal{S}|/N}$ shows that \textsc{Mimic-MD} also achieves suboptimality that has linear dependence on the length of the episode, albeit at the cost of worse dependence on $N$.
\end{enumerate}
Next we establish a lower bound on the error of any algorithm in the known-transition setting.

\begin{theoremi}[Formal version: \Cref{theorem:LB.p3}] \label{theorem:informal-5}
In the known-transition setting, for any learner $\widehat{\pi}$, there exists an instance such that the expected suboptimality $J(\pi^*) - \mathbb{E} [J (\widehat{\pi})] \gtrsim |\mathcal{S}|H/N$. This result applies even if the expert is constrained to be deterministic.
\end{theoremi}

It is important to note that our suboptimality lower bounds in \Cref{theorem:informal-1} corresponding to the no-interaction / active settings and \Cref{theorem:informal-5} in the known-transition setting are universal and apply for any learner's policy $\widehat{\pi}$. In contrast, the lower bound example in \cite{Ross-AIstats10} applies only for supervised learning. They construct a particular MDP and show that a particular learner strategy which plays an action different than the expert with probability $\epsilon$ has suboptimality $\gtrsim H^2 \epsilon$. In fact, it turns out that the suboptimality incurred by behavior cloning is exactly $0$ on the example provided in \cite{Ross-AIstats10}, \textit{given just a single expert trajectory}. Thus their result does not imply a uniform lower bound on the suboptimality of all learner algorithms as a function of the size of the dataset, $N$: even behavior cloning performs well on their example.

\section{Related Work} \label{section:related-work}

The classical approach to IL focuses on learning from fixed expert demonstrations, e.g., \cite{Abbeel-Ng-ILviaIRL, syed2008apprenticeship, ratliff2006maximum, ziebart2008maximum,finn2016guided, ho2016generative, pan2017agile}. The reduction approach has also received much attention for theoretical analysis of IL \cite{Ross-AIstats10,Brantley2020Disagreement-Regularized}. In the active setting, \cite{journals/jmlr/RossGB11} propose \textsc{Dagger}, \cite{Ross2014ReinforcementAI} propose \textsc{AggReVaTe}, and \cite{sun2017deeply} propose AggreVaTeD which learn policies by actively interact with the environment and the expert during training.  \cite{Luo2020Learning} propsose a value function approach that is able to self-correct in IL. IL has also received attention from the general approach of minimizing f-divergences \cite{ke2019imitation}. Very recently, \cite{arora2020provable} studies the imitation learning problem using a representation learning approach, where multiple agents' datasets are available for learning a common representation of the environment. While our results mainly focus on the case where both expert states and actions are observable, there are approaches e.g. \cite{ nair2017combining,  torabi2018behavioral,wen-sun-ILFO,arora2020provable}, studying the setting with observations of states alone. The statistical limits of IL in this setting is an interesting direction and is left as future work.

\section{Preliminaries} \label{section:preliminaries}

An MDP $\mathcal{M} = (\mathcal{S},\mathcal{A},\rho,P,\mathbf{r},H)$ describes the decision problem over state space $\mathcal{S}$ and action space $\mathcal{A}$. The initial state $s_1$ is drawn from a distribution $\rho$, and the state evolution at each time $t>1$ is specified by unknown transition functions, $P = \{ P_t (\cdot | s,a) : (s,a) \in \mathcal{S} \times \mathcal{A} \}_{t=1}^H$. In addition, there is an unknown reward function $\mathbf{r} = (\mathbf{r}_1,\cdots,\mathbf{r}_H)$ where each $\mathbf{r}_t : \mathcal{S} \times \mathcal{A} \to [0,1]$. Choosing the action $a$ at state $s$ at time $t$, returns the reward $\mathbf{r}_t (s,a)$. Interaction with the MDP happens by rolling out a policy $\pi$, which is a non-stationary mapping from states to distributions over actions. Namely, $\pi = (\pi_1,\cdots,\pi_H)$ where $\pi_t : \mathcal{S} \to \Delta_1 (\mathcal{A})$ and $\Delta_1 (\mathcal{A})$ is the probability simplex over $\mathcal{A}$. We operate in the episodic setting and recall that the value of a policy $\pi$, defined in \cref{eq:policy-value}, is the expected cumulative reward collected over an episode of length $H$. Here the $\mathbb{E}_\pi [\cdot]$ operator (resp. $\mathrm{Pr}_\pi [\cdot]$) defines expectation (resp. probability) computed with respect to the trajectory generated by rolling out $\pi$, namely $\{ s_1 \sim \rho ; \forall t \in [H], a_t \sim \pi_t (\cdot | s_t),\ s_{t+1} \sim P_t (\cdot | s_t,a_t) \}$.Som etimes we instead use $J_{\mathcal{M}} (\cdot)$ or $J_{\mathbf{r}} (\cdot)$ in order to make the underlying MDP or reward function explicit. We recap that the suboptimality $J (\pi^*) - J (\widehat{\pi})$ of the learner $\widehat{\pi}$ is defined as the difference in value of the expert and the learner.

Starting from the inital state $s_1 \sim \rho$, the learner \textit{interacts} with the MDP by sequentially choosing actions $a_t$ at visited states $s_t$, with the MDP transitioning the learner to the next state $s_{t+1}$ sampled from $P_t (\cdot | s_t,a_t)$. In the imitation learning framework, the reward function $\mathbf{r}_t (s_t,a_t)$ is unobserved at each time $t$ in an episode. However, the learner can access demonstrations from an expert $\pi^*$ with the objective of learning a policy that has value comparable to the expert.  We study imitation learning in the following $3$ settings:

\begin{enumerate}
\item[(a)] \textbf{No-interaction:} The learner is provided a dataset $D$ of $N$ trajectories drawn by independently rolling out the expert policy through the MDP. The learner is not otherwise allowed to interact with the MDP.

\item[(b)] \textbf{Known-transition:} As in the no-interaction setting, the learner is provided an expert dataset $D$ of $N$ trajectories drawn by rolling out the expert's policy. However, the learner additionally knows the MDP state transition functions $P_t$ throughout the episode, as well as the initial distribution over states $\rho$.

\item[(c)] \textbf{Active:} The learner is not given a dataset of expert demonstrations in advance. However, the learner is allowed to interact with the MDP for $N$ episodes and is provided access to an oracle which upon being queried, exactly returns the expert's action distribution $\pi^*_t (\cdot|s)$ at the learner's current state $s$.
\end{enumerate}

To facilitate the discussion when the expert is deterministic, we define the family of all deterministic policies by $\Pi_{\mathrm{det}}$. We use the notation $\pi^*_t (s)$ to denote the action played by a deterministic expert at state $s$ at time $t$.


\section{No-interaction setting} \label{section:no-interaction}

We first study the setting where the expert plays a deterministic policy $\pi^*$ and introduce the behavior cloning approach. The empirical $0$-$1$ risk is the empirical frequency of the learner choosing an action different from the expert, computed with the observed dataset $D$:
\begin{equation} \label{eq:0-1risk}
    \mathbb{I}_{\mathrm{emp}} (\widehat{\pi}, \pi^*) = \frac{1}{H} \sum\nolimits_{t=1}^H \mathbb{E}_{s_t \sim f_D^t} \Big[ \mathbb{E}_{a \sim \widehat{\pi}_t (\cdot |s_t)} \Big[ \mathbbm{1} (a \ne \pi^*_t (s_t)) \Big] \Big].
\end{equation}
Here $f_D^t$ is the empirical distribution over states at time $t$ averaged across trajectories in $D$. Note that a policy that carries out behavior cloning and minimizes the empirical $0$-$1$ risk to $0$ in fact mimics the expert at all states observed in $D$. Since the policy on the remaining states is not specified, we define $\Pi_{\mathrm{mimic}} (D)$ as the set of all candidate deterministic policies that carry out behavior cloning,
\begin{equation} \label{eq:Pi.mimic}
    \Pi_{\mathrm{mimic}} (D) \triangleq \Big\{ \pi \in \Pi_{\mathrm{det}} : \forall t \in [H], s \in \mathcal{S}_t (D),\ \pi_t (\cdot | s) = \delta_{\pi^*_t (s)} \Big\},
\end{equation}
where $\mathcal{S}_t (D)$ denotes the set of states visited at time $t$ in some trajectory in $D$, and $\pi_t^* (s)$ is the unique action played by the expert at time $t$ in any trajectory in $D$ that visits the state $s$ at time $t$. We also define the population $0$-$1$ risk of a policy as the probability that the learner chooses an action different from the expert under the state distribution induced by rolling out the expert's policy, $\pi^*$. In particular,
\begin{equation} \label{eq:pop-0-1risk}
    \mathbb{I}_{\mathrm{pop}} (\widehat{\pi}, \pi^*) = \frac{1}{H} \sum\nolimits_{t=1}^H \mathbb{E}_{s_t \sim f_{\pi^*}^t} \Big[ \mathbb{E}_{a \sim \widehat{\pi}_t (\cdot |s_t)} \Big[ \mathbbm{1} (a \ne \pi^*_t (s_t)) \Big] \Big].
\end{equation}
Here, $f_{\pi^*}^t$ is the distribution over states at time $t$ induced by rolling out the expert's policy $\pi^*$. Indeed, the reduction approach in \cite{Ross-AIstats10} shows that any policy $\widehat{\pi}$ that minimizes the population $0$-$1$ loss to be $\le \epsilon$ ensures that $J(\pi^*) - J (\widehat{\pi}) \le H^2 \epsilon$. We first analyze the behavior cloning approach which minimizes the empirical $0$-$1$ risk and establish a generalization bound for the expected population $0$-$1$ risk.

\begin{lemma}[Population $0$-$1$ risk of Behavior Cloning] \label{theorem:det:Nsim0:UB.p1}
Consider the no-interaction setting, and assume the expert's policy $\pi^*$ is deterministic. Consider any policy $\widehat{\pi} \in \Pi_{\mathrm{mimic}} (D)$ (defined in \cref{eq:Pi.mimic}) which is the set of policies that carry out behavior cloning. Then, the expected population $0$-$1$ risk of $\widehat{\pi}$ (defined in \cref{eq:pop-0-1risk}) is bounded by,
\begin{equation}
    \mathbb{E} \left[ \mathbb{I}_{\mathrm{pop}} (\widehat{\pi}, \pi^*) \right] \lesssim \min \left\{ 1, \frac{|\mathcal{S}|}{N} \right\}.
\end{equation}
\end{lemma}
\begin{proof}[Proof Sketch]
The bound on the population $0$-$1$ risk of behavior cloning relies on the following observation: at each time $t$, the learner exactly mimics the expert on the states that were visited in the expert dataset at least once. Therefore the contribution to the population $0$-$1$ risk only stems from states that were never visited at time $t$ in any trajectory in $D$. We identify that for each $t$, the probability mass contributed by such states has expected value upper bounded by $|\mathcal{S}| / N$. Plugging this back into the definition of the population $0$-$1$ risk completes the proof.
\end{proof}

With this result, invoking \cite[Theorem 2.1]{Ross-AIstats10} immediately results in the upper bound on the expected suboptimality of a learner carrying out behavior cloning in \Cref{theorem:det:Nsim0:UB.p2}. Furthremore, we use a similar approach to establish a high probability bound on the population $0$-$1$ risk of behavior cloning.

\begin{theorem}[Upper bounding suboptimality of Behavior Cloning]
Consider any policy $\widehat{\pi}$ which carries out behavior cloning (i.e. $\widehat{\pi} \in \Pi_{\mathrm{mimic}} (D)$).
\begin{enumerate}[label={(\alph*)},ref={\thetheorem~(\alph*)}]
    \item \label[theorem]{theorem:det:Nsim0:UB.p2} The expected suboptimality of $\widehat{\pi}$ is upper bounded by,
    \begin{equation}
        J (\pi^*) - \mathbb{E} \left[ J (\widehat{\pi}) \right] \lesssim \min \left\{ H, \frac{|\mathcal{S}| H^2}{N} \right\}.
    \end{equation}
    \item For any $\delta \in (0,\min \{ 1,H/10 \}]$, with probability $\ge 1-\delta$ the suboptimality of $\widehat{\pi}$ is bounded by,
    \begin{equation} \label{eq:det:Nsim0:high-prob:UB}
        J (\pi^*) - J (\widehat{\pi}) \lesssim \frac{|\mathcal{S}| H^2}{N} + \frac{\sqrt{|\mathcal{S}|} H^2 \log(H/\delta)}{N}.
    \end{equation}
    \label{theorem:det:Nsim0:UB.p3}
\end{enumerate}
\end{theorem}
\begin{proof}[Proof Sketch]
To establish the high probability bound on the population $0$-$1$ risk of behavior cloning, we utilize the key observation in the proof of \Cref{theorem:det:Nsim0:UB.p1}: for each $t=1,\cdots,H$, the contribution to the population $0$-$1$ risk in \cref{eq:pop-0-1risk} stems only from states that were never visited at time $t$ in any trajectory in $D$. For each $t$, we show that the mass contributed by such states up to constants does not exceed $\frac{|\mathcal{S}|}{N} + \frac{\sqrt{|\mathcal{S}|} \log (H/\delta)}{N}$ with probability $\ge 1 - \delta/H$. Summing over $t = 1,\cdots,H$ results in an upper bound on the population $0$-$1$ loss that holds with probability $\ge 1 - \delta$ (by the union bound). Invoking \cite[Theorem 2.1]{Ross-AIstats10} implies the high probability bound on $J(\pi^*) - J (\widehat{\pi})$.
\end{proof}

\subsection{Stochastic expert}

The behavior cloning approach indeed shows that when the expert is deterministic, carrying out supervised learning suffices to get good guarantees on the suboptimality. Motivated by the reduction in \cite[Theorem 2.1]{Ross-AIstats10}, a natural question to ask is whether a similar reduction to supervised learning applies when the expert is not restricted to be deterministic. Indeed, we show that when the expert plays a general policy, any learner which minimizes the TV distance to the expert's policy at states drawn by rolling out the expert has small suboptimality.

\subsubsection{Reduction of IL to supervised learning under TV distance} \label{subsec:reduction-generalized}

\cite[Theorem 2.1]{Ross-AIstats10} show that if the expert's policy is deterministic, and the probability of guessing the expert's action at each state is $\epsilon$, then $J(\pi^*) - J(\widehat{\pi}) \le \min \{ H, \epsilon H^2 \}$. In this section we prove a generalization of this result which applies even if the expert plays a stochastic policy. In particular, we consider a supervised learning reduction from imitation learning to matching the expert's policy in total variation (TV) distance. To this end, we first introduce the population TV risk,
\begin{equation} \label{eq:TVrisk}
    \mathbb{T}_{\mathrm{pop}} (\widehat{\pi}, \pi^*) = \frac{1}{H} \sum\nolimits_{t=1}^H \mathbb{E}_{s_t \sim f^t_{\pi^*}} \Big[ \mathsf{TV} \Big( \widehat{\pi}_t (\cdot |s_t), \pi^*_t (\cdot| s_t) \Big) \Big].
\end{equation}
We show that if the learner minimizes the population TV risk to be $\le \epsilon$ then the expected suboptimality of the learner is $\lesssim \min \{ H, H^2 \epsilon \}$. The population TV risk of a learner is a generalization of the population $0$-$1$ risk to the case where the expert's policy is stochastic.
We formally state the reduction below.
\begin{lemma} \label[lemma]{lemma:nondet-reduction}
Consider any policy $\widehat{\pi}$ such that $\mathbb{T}_{\mathrm{pop}} (\widehat{\pi}, \pi^*) \le \epsilon$. Then, $J(\pi^*) - J (\widehat{\pi}) \le \min \{ H, H^2 \epsilon \}$.
\end{lemma}

\begin{remark}
When the expert is deterministic, the definition of $\mathbb{T}_{\mathrm{pop}}$ matches that of $\mathbb{I}_{\mathrm{pop}}$ (\cref{eq:pop-0-1risk}) recovering the guarantee in \cite[Theorem 2.1]{Ross-AIstats10}. Thus, \Cref{lemma:nondet-reduction} strictly generalizes the supervised learning reduction for behavior cloning.
\end{remark}

While the reduction approach seems promising at first, there is a catch - the population TV risk in fact converges very slowly to $0$. Since it corresponds to matching the expert's and learner's action distributions, the convergence rate is $\gtrsim \sqrt{|\mathcal{A}|/N}$ even if $|\mathcal{S}|=1$. In the same setting, the population $0$-$1$ risk which is the counterpart in the deterministic expert setting converges at a much faster $\lesssim 1/N$ rate (Theorem~\ref{theorem:det:Nsim0:UB.p1}).

The significantly improved guarantees of the reduction approach when the expert is deterministic seem to suggest that imitation learning may be a significantly harder problem when the expert is stochastic. However, by circumventing the reduction framework, we show that this is in fact not the case. In \Cref{theorem:nondet:UB}, we show that expected suboptimality achieving the same $1/N$ rate of convergence (up to logarithmic factors) can be realized. This shows that the reduction analysis is in fact loose when the expert policy is stochastic.

\subsubsection{Circumventing the reduction approach}

In this section, we show that a natural policy in fact achieves a $1/N$ rate of convergence up to logarithmic factors. We consider \textsc{Mimic-Emp} (\Cref{alg:nondet:Nsim0}), which plays the empirical estimate of the expert's policy wherever available, and the uniform distribution over actions otherwise. Moreover the guarantees on expected suboptimality of this policy is optimal in the dependence on $H$ and achieves the error compounding lower bound in \Cref{theorem:LB.p1}. Note that the approach of playing the expert's empirical action distribution at states observed in the expert dataset in fact corresponds to minimizing the empirical risk under log-loss.

\begin{algorithm}
  \caption{\textsc{Mimic-Emp}} \label{alg:nondet:Nsim0}
  \begin{algorithmic}[1]
    \State \textbf{Input:} Expert dataset $D$
    \For{$t = 1,2,\cdots,H$}%
    \For{$s \in \mathcal{S}$}%
        \If{$s \in \mathcal{S}_t (D)$}
            \State $\widehat{\pi}_t (\cdot | s) = \pi^D_t (\cdot | s)$.
            \Comment{\parbox[t]{0.6\linewidth}{$\pi^D_t (\cdot | s)$ is the empirical estimate for $\pi^*_t (\cdot | s)$ in dataset $D$}}
        \Else
        \State $\widehat{\pi}_t (\cdot | s) = \mathrm{Unif} (\mathcal{A})$.
        \EndIf
    \EndFor
    \EndFor
    \State \textbf{Return} $\widehat{\pi}$
  \end{algorithmic}
\end{algorithm}

\begin{theorem} \label{theorem:nondet:UB}
In the no-interaction setting, the learner's policy $\widehat{\pi}$ returned by \textsc{Mimic-Emp} (\Cref{alg:nondet:Nsim0}) has expected suboptimality upper bounded by,
\begin{equation}
J (\pi^*) - \mathbb{E} \left[ J (\widehat{\pi}) \right] \lesssim \min \left\{ H, \frac{|\mathcal{S}| H^2 \log (N)}{N} \right\},
\end{equation}
for a general expert $\pi^*$ which could be stochastic. 
\end{theorem}

In contrast to the setting where the expert is determinstic, it is no longer true that the learner incurs no suboptimality as long as all states visited are observed in the expert dataset. However, by virtue of playing an empirical estimate of the expert's policy at these states it is plausible the expected suboptimality of the learner is $0$. However, a proof of this claim is not straightforward since the empirical distribution played by the learner at different states is not independent across time as functions of the dataset $D$.

We circumvent this problem by constructing a coupling between the expert's and learner's policies. Under the coupling it turns out the expected suboptimality of the learner is in fact $0$ when the visited states are all observed in the dataset. The remaining task is to bound the probability that at some point in the episode the learner visits a state unobserved in the expert dataset. A careful analysis of this probability term shows that it is bounded by $\lesssim |\mathcal{S}| H \log (N) / N$ under the coupling.

\section{Known-transition setting} \label{section:known-transition}

We next study imitation learning in the known-transition model where the initial state distribution and transition functions $P_t (\cdot | s,a)$ are known to the learner for all $s \in \mathcal{S}, a \in \mathcal{A}$ and $t \in [H]$. To indicate this, we denote the learner's policy by $\widehat{\pi} (D, P, \rho)$. In this setting, mimicking the expert on states where the expert's policy is known is still a good approach, since there is no contribution to the learner's suboptimality as long as the learner only visits such states in an episode. However compared to the no-interaction setting, with the additional knowledge of $P$, the learner can potentially do better on states that are not visited in the demonstrations, and \emph{correct} its mistakes even after it takes a wrong action, to avoid the error compounding problem.

\begin{theorem}
Consider the learner's policy $\widehat{\pi}$ returned by \textsc{Mimic-MD} (\Cref{alg:det:NsimInf}). When the expert policy $\pi^*$ is deterministic, in the known-transition setting,
\begin{enumerate}[label={(\alph*)},ref={\thetheorem~(\alph*)}]
    \item \label[theorem]{theorem:det:NsimInf:UB.p1} The expected suboptimality of the learner is upper bounded by,
    \begin{equation} \label{eq:det:NsimInf:UB.p1}
    J (\pi^*) - \mathbb{E} \left[ J (\widehat{\pi} (D, P,\rho)) \right] \lesssim \min \left\{ H,\ \sqrt{\frac{|\mathcal{S}|H^2}{N}} ,\ \frac{|\mathcal{S}|H^{3/2}}{N} \right\}.
    \end{equation}
    \item For any $\delta \in (0,\min\{ 1,H/5\})$, with probability $\ge 1 - \delta$, the suboptimality of the learner satisfies,
    \begin{equation} \label{eq:det:NsimInf:UB.p2}
        J(\pi^*) - J (\widehat{\pi}) \lesssim \frac{|\mathcal{S}| H^{3/2}}{N} \left( 1 + \frac{3\log (2|\mathcal{S}| H/\delta)}{\sqrt{|\mathcal{S}|}} \right)^{1/2} \sqrt{\log (2|\mathcal{S}| H/\delta)}.
    \end{equation}
    \label{theorem:det:NsimInf:UB.p2}
\end{enumerate}
\end{theorem}

\begin{algorithm}[t]
  \caption{\textsc{Mimic-MD}}
  \label{alg:det:NsimInf}
  \begin{algorithmic}[1]
    \State \textbf{Input:} Expert dataset $D$.
    \State Choose a uniformly random permutation of $D$,
    \Statex Define $D_1$ to be the first $N/2$ trajectories of $D$ and $D_2 = D \setminus D_1$.
    \State Define $\mathcal{T}^{D_1}_t ( s,a) \triangleq \{ \{ (s_{t'},a_{t'}) \}_{t'=1}^H | s_t{=}s, a_t{=}a,\ \exists \tau \le t : s_\tau \not\in \mathcal{S}_\tau (D_1) \}$ as trajectories that visit $(s,a)$ at time $t$, and at some time $\tau \le t$ visit a state unvisited at time $\tau$ in any trajectory in $D_1$.
    \State Define the following optimization problem:
    \begin{equation} \label{eq:opt}
        \argmin_{\pi \in \Pi_{\mathrm{mimic}} (D_1)} \ \sum_{t=1}^H \sum_{(s,a) \in \mathcal{S} \times \mathcal{A}} \left| \mathrm{Pr}_\pi \Big[ \mathcal{T}^{D_1}_t ( s,a) \Big] - \frac{\sum_{\textsf{tr} \in D_2} \mathbbm{1} \left( \textsf{tr} \in \mathcal{T}^{D_1}_t ( s,a) \right)}{|D_2|} \right| \tag{\textsf{OPT}}
    \end{equation}
    Choose $\widehat{\pi}$ as any optimizer of \ref{eq:opt}.
    \Statex \Comment{$\Pi_{\mathrm{mimic}} (D_1)$ is the set of policies that mimics the expert on the states visited in $D_1$ (\cref{eq:Pi.mimic})}
    \State \textbf{Return} $\widehat{\pi}$
  \end{algorithmic}
\end{algorithm}

\Cref{theorem:det:NsimInf:UB.p1} shows that \textsc{Mimic-MD} (\Cref{alg:det:NsimInf}) breaks the $|\mathcal{S}|H^2/N$ error compounding barrier which is not possible in the no-interaction setting, as discussed later in \Cref{theorem:LB.p1}. \textsc{Mimic-MD} inherits the spirit of mimicking the expert by exactly copying the expert actions in dataset $D_1$: as a result, the learner only incurs suboptimality upon visiting a state unobserved in $D_1$ at some point in an episode. Let $\mathcal{E}_{D_1}^{\le t}$ be the event that the learner visits a state at some time $\tau \le t$ which has not been visited in any trajectory in $D_1$ at time $\tau$. In particular, for any policy $\widehat{\pi}$ which mimics the expert on $D_1$, we show,
\begin{equation} \label{eq:NsimInf:motiv}
    J(\pi^*) - J (\widehat{\pi}) \le \sum_{s \in \mathcal{S}} \sum_{a \in \mathcal{A}} \sum\nolimits_{t=1}^H \left| \mathrm{Pr}_{\pi^*} \Big[ \mathcal{E}_{D_1}^{\le t} ,s_t=s,a_t=a \Big] - \mathrm{Pr}_{\widehat{\pi}} \Big[ \mathcal{E}_{D_1}^{\le t}, s_t=s,a_t=a \Big] \right|.
\end{equation}
In the known-transition setting the learner knows the transition functions $\{ P_t : 1 \le t \le H \}$ and the initial state distribution $\rho$, and can exactly compute the probability $\mathrm{Pr}_{\pi} [ \mathcal{E}_{D_1}^{\le t} ,s_t=s,a_t=a ]$ for any known policy $\pi$. However, unfortunately the learner cannot compute $\mathrm{Pr}_{\pi^*} [ \mathcal{E}_{D_1}^{\le t} ,s_t=s,a_t=a ]$ given only $D_1$. This is because the expert's policy on states unobserved in $D_1$ is unknown and the event $\mathcal{E}_{D_1}$ ensures that such states are necessarily visited. Here we use the remaining trajectories in the dataset, $D_2$ to compute an empirical estimate of $\mathrm{Pr}_{\pi^*} [ \mathcal{E}_{D_1}^{\le t} ,s_t=s,a_t=a ]$. The form of \cref{eq:NsimInf:motiv} exactly motivates \Cref{alg:det:NsimInf}, which replaces the population term $\mathrm{Pr}_{\pi^*} [ \mathcal{E}_{D_1}^{\le t} ,s_t=s,a_t=a ]$ by its empirical estimate in the MD functional.

\begin{remark}
In the known-transition setting, the maximum likelihood estimate (MLE) for $\pi^*$ does not achieve the optimal sample complexity. When the expert is deterministic, all policies in $\Pi_{\mathrm{mimic}} (D)$ have equal likelihood given $D$. This is because the probability of observing a trajectory does not depend on the expert's policy on the states it does not visit. From \Cref{theorem:det:Nsim0:UB.p2} and \Cref{theorem:LB.p1} the expected suboptimality of the worst policy in $\Pi_{\mathrm{mimic}} (D)$ is $\asymp |\mathcal{S}|H^2 / N$. Since the MLE does not give a rule to break ties, this implies that it is not optimal.
\end{remark}

\begin{remark}
The standard analysis of conventional minimum distance functional / distribution matching approaches rely on convergence of the empirical distribution to the population in the corresponding distance functional. For most non-trivial choices of the distance functional, this convergence rate is slow and is $\gtrsim 1/\sqrt{N}$, given $N$ samples. At a technical level, the state distributions are matched only at states unvisited in the expert dataset. In particular, $1/N$ rate of convergence of \textsc{Mimic-MD} relies on the fact that the effective mass of the distributions being matched shrinks from $1$ to $|\mathcal{S}|H/N$.
\end{remark}

\begin{remark}
Although data splitting may not be necessary, we conjecture that the conventional minimum distance functional approach, which matches the empirical distribution of either states or state-action pairs does not achieve the rate in \Cref{theorem:det:NsimInf:UB.p1} since it does not necessarily exactly mimic the expert on the observed demonstrations. In particular, conventional distribution matching approaches do not take into account the fact that the expert’s action is known at every state visited in the dataset. These policies may choose to play a different action at a state, even if the expert’s action is observed in the dataset. In contrast, \textsc{Mimic-MD} returns a policy that is constrained to mimic the expert at states visited in the expert dataset, avoiding this issue.
\end{remark}

\begin{remark} \label{remark:1}
The optimization problem \ref{eq:opt} solved by \textsc{Mimic-MD} is over multivariate degree-$H$ polynomials in $\{ \pi_1 (\cdot | \cdot) , \cdots, \pi_H (\cdot | \cdot) \}$. In general it is not possible to solve this optimization problem in polynomial (in $H$) time. However, we appeal to the fact that the polynomial is sparse having at most $N$ non-zero coefficients. Moreover, our analysis does not require that the optimization problem \ref{eq:opt} be solved exactly, which we discuss in \Cref{theorem:det:NsimInf:UB.p3,remark:3}.
\end{remark}

We also provide a guarantee when the learner solves the optimization problem in \ref{eq:opt} to an accuracy of $\varepsilon$. The guarantee on suboptimality admit by \textsc{Mimic-MD} in \Cref{theorem:det:NsimInf:UB.p1} is recovered taking $\varepsilon = 0$.

\begin{remark}
\label{theorem:det:NsimInf:UB.p3}
Consider any policy $\widehat{\pi}$ that minimizes the optimization problem \ref{eq:opt} to an additive error of $\varepsilon$. Then, the expected suboptimality of the learner is upper bounded by,
\begin{equation} \label{eq:det:NsimInf:UB.p3}
J (\pi^*) - \mathbb{E} \left[ J (\widehat{\pi} (D, P)) \right] \lesssim \min \left\{ H,\ H\sqrt{\frac{|\mathcal{S}|}{N}} + \varepsilon,\ \frac{|\mathcal{S}|H^{3/2}}{N} + \varepsilon \right\}.
\end{equation}
\end{remark}

\begin{remark} \label{remark:3}
\Cref{theorem:det:NsimInf:UB.p3} shows that \textsc{Mimic-MD} is amenable in the following settings and combinations thereof.
\begin{enumerate}
    \item As discussed in \Cref{remark:1}, optimization problems over multivariate degree $H$ polynomials (as in \ref{eq:opt}) in general are not exactly solvable in polynomial time. \Cref{theorem:det:NsimInf:UB.p3} shows that it suffices to solve \ref{eq:opt} approximately to result in a policy with small suboptimality.
    \item This approach applies in the approximate transition setting, where the transition functions are not known exactly but are known approximately. In particular, suppose the learner's policy $\widehat{\pi}$ solves \ref{eq:opt} exactly when the probabilities $\mathrm{Pr}_\pi [\cdot]$ are computed using the approximate transition functions. By the smoothness of $\mathrm{Pr}_\pi [\cdot]$ one can bound the suboptimality of $\widehat{\pi}$ on \ref{eq:opt} when the probabilities are instead computed using exact transition functions. Applying \Cref{theorem:det:NsimInf:UB.p3} for this $\varepsilon$ controls the suboptimality of the resulting policy.
\end{enumerate}
\end{remark}

\section{Lower bounds} \label{section:lower-bounds}

In this section we discuss lower bound constructions in the no-interaction, active and known-transition settings. Our first contribution is a lower bound on the expected suboptimality of any policy in the no-interaction and active settings.

\begin{theorem} \label{theorem:LB.p1}
For any learner $\widehat{\pi}$, there exists an MDP $\mathcal{M}$ and a deterministic expert policy $\pi^*$ such that the expected suboptimality of the learner is lower bounded in the no-interaction setting by,
\begin{equation}
    J_{\mathcal{M}} (\pi^*) - \mathbb{E} [J_{\mathcal{M}} (\widehat{\pi})] \gtrsim \min \left\{ H, |\mathcal{S}| H^2 / N \right\}.
\end{equation}
Furthermore this lower bound applies even when the learner operates in the active setting.
\end{theorem}

\noindent We construct the worst case MDP templates for the no-interaction and active settings in \Cref{fig:det:Nsim0:LB} and that for the known-transition setting in \Cref{fig:det:NsimInf:LB} and defer the formal analysis to the appendix.\\

\noindent In \Cref{fig:det:Nsim0:LB}, at any state of the MDP, except one, every other action, moves the learner to the absorbing state $b$. Suppose a learner independently plays an action different from the expert at a state with probability $\epsilon$. Upon making a mistake, the learner is transferred to $b$ and collects no reward for the rest of the episode. Thus the suboptimality of the learner is $\ge H \epsilon + (H-1) \epsilon (1 - \epsilon) +\cdots+ (1-\epsilon)^H \gtrsim \min \{ H, H^2 \epsilon \}$. By construction of $\rho$, we identify that any learner must make a mistake with probability $\epsilon \gtrsim  |\mathcal{S}|/N$, resulting in the claim. It is interesting to observe that this argument closely resembles the intuition mentioned in the introduction for the $\lesssim \epsilon H^2$ bound on suboptimality that the reduction to supervised learning guarantees. In the following remark, we address the active setting and show that the same lower bound construction carries over.

\begin{remark}
The lower bound construction in \Cref{fig:det:Nsim0:LB} applies even if the learner can query the expert while interacting with the MDP. If the expert's queried action is not followed at any state, the learner is transitioned to $b$ with probability $1$. Upon doing so, the learner no longer can get any meaningful information about the expert's policy at states for the rest of the episode. Seeing that the ``most informative'' dataset the learner can collect involves following the expert at each time, it is no different had an expert dataset of $N$ trajectories been provided in advance. This reduces the active case to the no-interaction case for which the existing construction applies.
\end{remark}

\begin{remark}
In the no-interaction setting, \Cref{theorem:LB.p1} in conjunction with \Cref{theorem:det:Nsim0:UB.p2} shows that when the expert plays a deterministic policy, behavior cloning achieves the optimal expected suboptimality of $\frac{|\mathcal{S}|H^2}{N}$ for imitation learning. Furthermore, this shows the optimality of behavior cloning even in the active learning setting as well. Thus, the ability to actively query the expert does not improve the sample complexity of imitation learning when the expert is deterministic. Lastly, in case the expert's policy is stochastic, in conjunction with \Cref{theorem:nondet:UB} this result shows that \textsc{Mimic-Emp} is optimal upto logarithmic factors in $N$ in the no-interaction setting.
\end{remark}

We next lower bound the expected suboptimality incurred by any learner's policy in the known-transition setting and provide a short intuition below.

\begin{theorem}
\label{theorem:LB.p3}
In the known transition setting, for any learner $\widehat{\pi} (D,P,\rho)$, there exists an MDP $\mathcal{M}$ and a deterministic expert policy $\pi^*$ such that the expected suboptimality of the learner is lower bounded by, $J_{\mathcal{M}} (\pi^*) - \mathbb{E} [J_{\mathcal{M}} (\widehat{\pi})] \gtrsim \min \left\{ H, |\mathcal{S}| H / N \right\}$.
\end{theorem}

The lower bound instance in this construction is provided in \Cref{fig:det:NsimInf:LB}. In these MDPs, each state is absorbing so a policy only stays at a single state for the whole episode. If the initial state of the MDP was not visited in the dataset, the learner does not see the expert's action for the rest of the episode which is the only one at each state to offer non-zero reward. Conditioned on being initialized at such a state, the expected suboptimality of the learner is $\gtrsim H$. By construction of $\rho$, we determine that probability of the learner starting in such a state is $\gtrsim |\mathcal{S}| / N$ in expectation over the expert dataset $D$, resulting in the claim.

\begin{remark}
Our suboptimality lower bounds in \Cref{theorem:LB.p1} for the no-interaction / active settings and \Cref{theorem:LB.p3} for the known-transition setting are universal - they apply for any learner's policy $\widehat{\pi}$. In contrast, the lower bound example in \cite{Ross-AIstats10} (see Figure 1 and related discussion in their paper) applies only for supervised learning and is not universal. They construct a particular MDP and show that there exists a particular learner policy which (i) plays an action different than the expert with probability $\epsilon$, and (ii) suboptimality $\gtrsim H^2 \epsilon$. In fact on this example, the suboptimality of behavior cloning is exactly $0$ if the learner is provided even \textit{a single expert trajectory}. Thus their example does not provide a lower bound on the suboptimality of all learner algorithms as a function of the size of the dataset, $N$. In particular, even behavior cloning performs well on their example.
\end{remark}

\begin{remark}
In the known-transition setting, this lower bound in conjunction with \Cref{theorem:det:NsimInf:UB.p1} shows that when the expert is deterministic, then \textsc{Mimic-MD} (\Cref{alg:det:NsimInf}) is optimal in the dependence on $|\mathcal{S}|$ and $N$ and is suboptimal by a factor of at most $\sqrt{H}$ in its dependence on the episode length.
\end{remark}

\section{Conclusion} \label{section:conclusion}
We show that behavior cloning is in fact optimal in the no-interaction setting, when the expert is determinstic. In addition, we show that minimizing empirical risk under log-loss results in a policy which is optimal up to logarithmic factors even when the expert is stochastic. In the known-transition setting we propose the first policy that provably breaks the $H^2$ error compounding barrier, and show a lower bound which it matches up to a $\sqrt{H}$-factor. An important question we raise is to bridge this gap between the upper and lower bounds in the known-transition setting. In addition, we study IL at two opposing ends of the spectrum: when the learner cannot interact with the MDP, and when the learner exactly knows the underlying transition structure. Although there is at most a factor $H$ gap between these two settings, it is a fundamental question to ask is how much improvement is possible when the learner is allowed to interact with the MDP a finite number of times. It is also an interesting question to extend these results beyond the tabular setting.

\begin{figure}[t]
\centering
\begin{subfigure}[t]{0.49\textwidth}
    \captionsetup{format=hang}
    \centering
    \begin{tikzpicture}[shorten >=1pt,node distance=1.25cm,on grid,auto,good/.style={circle, draw=black!60, fill=green!30!white, thick, minimum size=9mm, inner sep=0pt},bad/.style={circle, draw=black!60, fill=red!30, thick, minimum size=9mm, inner sep=0pt}]
    \tikzset{every edge/.style={very thick, draw=red!30!white}}
    \tikzset{every loop/.style={min distance=9mm,in=86,out=125, looseness=8, very thick, draw=green!75!black}}
    
    \node[good]     (s_1)                   {$1$};
    \node           (dist1) [above = 1.5cm of s_1]  {${\sim}\rho$};
    \node           (dot)   [right of=s_1]  {$\cdots$};
    \node[good]     (s_2)   [right of=dot]  {\tiny $|\mathcal{S}|{-}1$};
    \node           (dist2) [above = 1.5cm of s_2]  {${\sim}\rho$};
    \node           (s_n)   [right of=s_2]  {};
    \node[bad]      (b)     [right of=s_n]  {$b$};

    \path[->,>=stealth]
    (s_1) edge [draw=green!75!black] node {} (dist1)
          edge [in = 147.5, out = 35]  node {} (b)
          edge [in = 160, out = 20]  node {} (b)
    (s_2) edge [draw=green!75!black] node {} (dist2)
          edge [in = 172.5, out = 15]  node {} (b)
          edge [in = 185, out = -7.5]  node {} (b)
    (b)   edge [loop above] node {} ()
    (b)   edge [distance=10mm,in=55,out=20, very thick] node {} (b);
    \end{tikzpicture}
    \caption{MDP template when in the no-interaction setting,\\ Upon playing the expert's action at any state except $b$, learner is renewed in the initial distribution,\\ $\rho = \{ \zeta,{\cdots},\zeta,1 {-} (|\mathcal{S}|{-}2) \zeta, 0\}$ where $\zeta {=} \frac{1}{N+1}$}
    \label{fig:det:Nsim0:LB}
\end{subfigure}%
~
\begin{subfigure}[t]{0.48\textwidth}
    \centering
    \begin{tikzpicture}[shorten >=1pt,node distance=1.25cm,on grid,auto,good/.style={circle, draw=black!60, fill=green!30!white, thick, minimum size=8mm, inner sep=0pt},bad/.style={circle, draw=black!60, fill=red!30, thick, minimum size=8mm, inner sep=0pt}]
    \tikzset{every edge/.style={very thick, draw=red!30!white}}
    \tikzset{every loop/.style={min distance=10mm,in=70,out=30,looseness=8, very thick, draw=green!75!black}}
    
    \node[good]     (s_1)                   {$1$};
    \node           (dot)   [right=of s_1]  {$\cdots$};
    \node[good]     (s_n)   [right of=dot]  {\small $|\mathcal{S}|$};

    \path[->,>=stealth]
    (s_1) edge [distance=10.5mm,in=110,out=150, very thick] node {} (s_1)
    (s_1) edge [loop above] node {}     ()
    (s_n) edge [distance=10.5mm,in=110,out=150, very thick] node {} (s_n)
    (s_n) edge [loop above] node {}     ();
    \end{tikzpicture}
    \caption{MDP template in the known-transition setting,\\ Each state is absorbing, initial distribution is given by $\{ \zeta,{\cdots},\zeta, 1 - (|\mathcal{S}|{-}1)\zeta \}$ where $\zeta = \frac{1}{N+1}$}
    \label{fig:det:NsimInf:LB}
\end{subfigure}
\caption{MDP templates for lower bounds under different settings: green arrows indicate state transitions under the expert's action, red arrows indicate state transitions under other actions}
\end{figure}
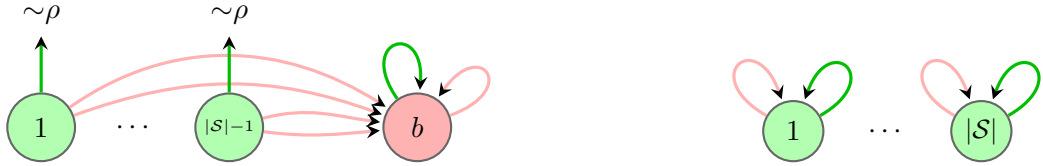

\bibliographystyle{plainnat}
\bibliography{refs}

\begin{appendices}

\section{} \label{app:a}
\localtableofcontents

We provide proofs for the theorems introduced previously in this appendix. We push the proofs of some of the lemmas and claims invoked in this section to Appendix~\ref{app:b}. For the remainder of the paper, we use $\log (\cdot)$ to denote the natural logarithm.

\subsection{No-interaction setting under deterministic expert policy}

In this section, we discuss the no-interaction setting where the learner is provided access to a dataset $D$ of $N$ trajectories generated by rolling out the expert's policy $\pi^*$, and is otherwise not allowed to interact with the MDP. Our goal is to provide guarantees on the expected suboptimality of a policy that carries out behavior cloning when the expert's policy is deterministic. As stated previously, we realize this guarantee by first bounding the population $0$-$1$ risk of behavior cloning (Theorem~\ref{theorem:det:Nsim0:UB.p1}) and then invoking the black box reduction guarantee from \cite{Ross-AIstats10}.

\subsubsection{Analysis of expected suboptimality of behavior cloning}

We first discuss the proof of \Cref{theorem:det:Nsim0:UB.p1} and~\Cref{theorem:det:Nsim0:UB.p2}, which bounds the expected suboptimality of a policy carrying out behavior cloning, assuming the expert's policy is deterministic.

Recall that the population $0$-$1$ loss is defined as,
\begin{align}
    \mathbb{I}_{\mathrm{pop}} (\widehat{\pi}, \pi^*) = \frac{1}{H} \sum\nolimits_{t=1}^H \mathbb{E}_{s_t \sim f_{\pi^*}^t} \Big[ \mathbb{E}_{a \sim \widehat{\pi}_t (\cdot |s_t)} \Big[ \mathbbm{1} (a \ne \pi^*_t (s)) \Big] \Big].
\end{align}
where $f^t_{\pi^*}$ is the state distribution induced at time $t$ rolling out the expert's policy $\pi^*$. We consider a learner $\widehat{\pi}$ that carries out behavior cloning given the expert dataset $D$ in advance. In particular, the learner's policy $\widehat{\pi}$ is a member of $\Pi_{\mathrm{mimic}} (D)$ since it exactly mimics the expert on the states that were visited at each time in some trajectory in the expert dataset. Thus the contribution to the population $0$-$1$ risk comes from the remaining states $s \in \mathcal{S}_t (D)$,
\begin{align}
    \mathbb{I}_{\mathrm{pop}} (\widehat{\pi}, \pi^*) &\le \frac{1}{H} \sum\nolimits_{t=1}^H \mathbb{E}_{s_t \sim f_{\pi^*}^t} \Big[ \mathbbm{1} (s_t \not\in \mathcal{S}_t (D)) \Big], \\
    &= \frac{1}{H} \sum\nolimits_{t=1}^H \sum\nolimits_{s \in \mathcal{S}} \mathrm{Pr}_{\pi^*} [s_t = s] \mathbbm{1} (s \not\in \mathcal{S}_t (D)). \label{eq:preexp}
\end{align}
Taking expectation on both sides gives,
\begin{equation}
    \mathbb{E} \left[ \mathbb{I}_{\mathrm{pop}} (\widehat{\pi}, \pi^*) \right] \le \frac{1}{H} \sum\nolimits_{t=1}^H \sum\nolimits_{s \in \mathcal{S}} \mathrm{Pr}_{\pi^*} [s_t = s] \mathrm{Pr} (s \not\in \mathcal{S}_t (D)). \label{eq:BC-upper-bound}
\end{equation}
In \Cref{lemma:BC-prob-small} we show that this expression is bounded by $\lesssim |\mathcal{S}| / N$, which completes the proof of the population $0$-$1$ risk bound of behavior cloning in Theorem~\ref{theorem:det:Nsim0:UB.p1}.

\begin{lemma} \label[lemma]{lemma:BC-prob-small}
$\mathbb{E} \left[ \sum_{t=1}^H \sum_{s \in \mathcal{S}} \mathrm{Pr}_{\pi^*} [s_t = s] \ \mathrm{Pr} [ s \not\in \mathcal{S}_t (D)] \right] \le \frac{4}{9} \frac{|\mathcal{S}| H}{|D|}$.
\end{lemma}

As stated previously, we subsequently invoke the supervised learning reduction in \cite[Theorem 2.1]{Ross-AIstats10} to upper bound the expected suboptimality of a learner carrying out behavior cloning in Theorem~\ref{theorem:det:Nsim0:UB.p2}.

The previous discussion is also amenable for establishing a high probability bound on the expected suboptimality of behavior cloning. Indeed, consider the upper bound on the population $0$-$1$ risk of behavior cloning in \cref{eq:preexp}, which is a function of the expert dataset $D$. It captures the probability mass under the expert's state distribution contributed by states unobserved in the expert dataset.

\subsubsection{High probability bounds for behavior cloning}

It turns out that the contribution to the upper bound on population $0$-$1$ risk of behavior cloning in \cref{eq:preexp} is captured by the notion of ``missing mass'' of the time-averaged state distribution under the expert's policy. The high probability result for behavior cloning (Theorem~\ref{theorem:det:Nsim0:UB.p3}) follows shortly by invoking existing concentration bounds for missing mass.

\begin{definition}[Missing mass]
Consider some distribution $\nu$ on $\mathcal{X}$, and let $X^N \overset{\text{i.i.d.}}{\sim} \nu$ be a dataset of $N$ samples drawn i.i.d. from $\nu$.
Let $\mathfrak{n}_x (X^N) = \sum_{i=1}^N \mathbbm{1} (X_i = x)$ be the number of times the symbol $x$ was observed in $X^N$. Then, the missing mass $\mathfrak{m}_0 (\nu, X^N) = \sum_{x \in \mathcal{X}} \nu (x) \mathbbm{1} (\mathfrak{n}_x (X^N) = 0)$ is the probability mass contributed by symbols never observed in $X^N$.
\end{definition}

It turns out that the missing mass of an arbitrary discrete distribution admits sub-Gaussian concentration. Invoking \cite[Lemma 11]{McAllesterOrtiz} establishes the following concentration guarantee for missing mass. A proof of the result is provided in Appendix~\ref{app:b}.

\begin{theorem} \label{theorem:missingmass-conc}
Consider an arbitrary distribution $\nu$ on $\mathcal{X}$, and let $X^N \overset{\text{i.i.d.}}{\sim} \nu$ be a dataset of $N$ samples drawn i.i.d. from $\nu$. Consider any $\delta \in (0,1/10]$. Then,
\begin{equation}
    \Pr \left( \mathfrak{m}_0 (\nu, X^N) - \mathbb{E} [\mathfrak{m}_0 (\nu, X^N)] \ge \frac{3\sqrt{|\mathcal{X}|} \log (1/\delta)}{N} \right) \le \delta.
\end{equation}
\end{theorem}

Consider the upper bound to the population $0$-$1$ loss in \cref{eq:BC-upper-bound}. Observe that for each fixed $\tau \in [H]$, $\sum_{s \in \mathcal{S}} \mathrm{Pr}_{\pi^*} [ s_\tau = s ] \mathbbm{1} (s \not\in \mathcal{S}_\tau (D))$ is the missing mass of $f_{\pi^*}^\tau$, given $N$ samples from the distribution. Recall that $f_{\pi^*}^\tau$ is the distribution over states at time $\tau$ rolling out $\pi^*$. Thus we can invoke the concentration bound from \Cref{theorem:missingmass-conc} to prove that the upper bound on $0$-$1$ loss in \cref{eq:BC-upper-bound} concentrates. We formally state this result in \Cref{lemma:error-prob-conc}.

\begin{lemma} \label[lemma]{lemma:error-prob-conc}
For any $\delta$ such that $\delta \in (0, \min \{1, H/10\}]$, with probability $\ge 1 - \delta$ over the randomness of the expert dataset $D$,
\begin{equation}
    \frac1H \sum\nolimits_{\tau = 1}^H \sum\nolimits_{s \in \mathcal{S}} \mathrm{Pr}_{\pi^*} [ s_\tau = s ] \mathbbm{1} (s \not\in \mathcal{S}_\tau (D)) \le \frac{4 |\mathcal{S}|}{9 N} + \frac{3\sqrt{|\mathcal{S}|} \log(H/\delta)}{N}.
\end{equation}
\end{lemma}

Plugging this result into \cref{eq:BC-upper-bound} provides an upper bound on the population $0$-$1$ risk of behavior cloning. Subsequently invoking \cite[Theorem 2.1]{Ross-AIstats10}, we arrive at the high probability bound on $J(\pi^*) - J(\widehat{\pi})$ for behavior cloning in \Cref{theorem:det:Nsim0:UB.p3}.

\subsection{No-interaction setting when the expert policy is stochastic}

In this section we continue to discuss the no-interaction setting, but drop the assumption that the expert plays a deterministic policy. We assume the expert plays a general stochastic policy.

\subsubsection{Analyzing expected suboptimality of \textsc{Mimic-Emp}}

In this section we discuss the proof of \Cref{theorem:nondet:UB} which bounds the expected suboptimality of \textsc{Mimic-Emp}.
Recall that the objective is to upper bound $J(\pi^*) - \mathbb{E} [J (\widehat{\pi})]$ when the learner carries out \textsc{Mimic-Emp}. The outline of the proof is to construct two policies $\pi^{\mathrm{first}}$ and $\pi^{\mathrm{orc-first}}$ that are functions of the dataset $D$.

The policy $\pi^{\mathrm{first}}$ is easy to describe: order the expert dataset arbitrarily, and at a state, play the action in the first trajectory in $D$ that visits it, if it exists. If no such trajectories exist, the policy plays $\mathrm{Unif} (\mathcal{A})$. In particular, we show that the value of $\pi^{\mathrm{first}}$ and \textsc{Mimic-Emp} are the same, taking expectation over the expert dataset $D$ (\Cref{lemma:val-equal}).

On the other hand, we consider an oracle policy $\pi^{\mathrm{orc-first}}$ which is very similar. Indeed, $\pi^{\mathrm{orc-first}}$ first orders the expert dataset in the same manner as $\pi^{\mathrm{first}}$. At any state, it too plays the action in the first trajectory in $D$ that visits it, if it exists. However, if such a trajectory does not exist, $\pi^{\mathrm{orc-first}}$ simply samples an action from the expert's action distribution and plays it at this state. This explains the namesake of the policy, since it requires oracle access to the expert's policy. By virtue of choosing actions this way, we show that the value of $\pi^{\mathrm{orc-first}}$ in expectation equals $J(\pi^*)$ (\Cref{lemma:expert=orcfirst}).

At an intuitive level the elements of the proof seem to be surfacing: $\pi^{\mathrm{orc-first}}$ matches $\pi^*$ in value, but is not available to the learner. However, it shares a lot of similarity to $\pi^{\mathrm{first}}$, which in expectation matches $\widehat{\pi}$ in value, the policy we wish to analyze. Informally,
\begin{equation}
    \widehat{\pi} \iff \pi^{\mathrm{first}} \approx \pi^{\mathrm{orc-first}} \iff \pi^*.
\end{equation}
Thus to establish the bound, we carry out an analysis of $J(\pi^{\mathrm{orc-first}}) - J(\pi^{\mathrm{first}})$. Indeed we show that since the two policies are largely the same, the learner is suboptimal only on the trajectories where at some point a state is visited where the policies do not match. The final element of the proof is to show that this event in fact occurs with low probability given an expert dataset of sufficiently large size.

Before delving into the formal definitions of $\pi^{\mathrm{first}}$ and $\pi^{\mathrm{orc-first}}$ and other elements of the proof, we introduce a modicum of relevant notation.

\paragraph{Notation:} We assume that the trajectories in the expert dataset $D$ are ordered arbitrarily as $\{ \textsf{tr}_1,\cdots,\textsf{tr}_N \}$. In addition, we denote each trajectory $\textsf{tr}_n$ explicitly as $\{ (s_1^n,a_1^n),\cdots, (s_H^n,a_H^n) \}$. For each state $s \in \mathcal{S}$ we define,
\begin{equation} \label{eq:Nts}
    N_{t,s} = \left\{ n \in [N] : s_t^n = s \right\},
\end{equation}
as the (totally) ordered set of indices of trajectories in $D$ which visit the state $s$ at time $t$.

The policy $\widehat{\pi}$ returned by \textsc{Mimic-Emp} samples an action from the empirical estimate of the expert's policy at each state wherever available. On the remaining states, the learner plays the distribution $\mathrm{Unif} (\mathcal{A})$.

Given the ordered dataset $D$, we define the policy $\pi^{\mathrm{first}} (D)$ as,
\begin{equation} \label{eq:def:pi-first}
    \pi^{\mathrm{first}}_t (\cdot | s) =
    \begin{cases}
    \delta_{a_t^n} \qquad &\text{if } |N_{t,s}| \ge 1, \text{ where } n = \min (N_{t,s}),\\
    \mathrm{Unif} (\mathcal{A}) \qquad &\text{otherwise.}
    \end{cases}
\end{equation}
In other words, $\pi^{\mathrm{first}} (D)$ plays the action in the first trajectory that visits the state $s$ at time $t$.

In order to analyze the expected suboptimality of $\widehat{\pi} (D)$, we first show that $\widehat{\pi} (D)$ and $\pi^{\mathrm{first}} (D)$ have the same value in expectation, and instead study the policy $\pi^{\mathrm{first}} (D)$.

\begin{lemma} \label[lemma]{lemma:val-equal}
$\mathbb{E} [J (\widehat{\pi} (D))] = \mathbb{E} [J (\pi^{\mathrm{first}} (D))]$.
\end{lemma}
With this result, we can write the expected suboptimality of the learner $\widehat{\pi}$ as,
\begin{equation} \label{eq:emp-to-first}
    J(\pi^*) - \mathbb{E} [J ( \widehat{\pi} (D))] = J(\pi^*) - \mathbb{E} [J ( \pi^{\mathrm{first}} (D))].
\end{equation}

We next move on to the discussion of $\pi^{\mathrm{orc-first}}$ which is an oracle version of $\pi^{\mathrm{first}}$. Informally, at any state $\pi^{\mathrm{orc-first}}$ plays the action from the first trajectory that visits it in $D$, if available. However on the remaining states instead of playing $\mathrm{Unif} (\mathcal{A})$, $\pi^{\mathrm{orc-first}}$ samples an action from the expert's action distribution and plays it at this state. Thus, $\pi^{\mathrm{orc-first}}$ is coupled with the expert dataset $D$.

Prior to discussing $\pi^{\mathrm{orc-first}}$ in greater depth, we first introduce some preliminaries. In particular, we adopt an alternate view of the process generating the expert dataset $D$ which will play a central role in formally defining $\pi^{\mathrm{orc-first}}$. We mention that this approach is inspired by the alternate view of Markov processes in \cite{billingsley1961}.

To this end, we first define an ``expert table'' which is a fixed infinite collection of actions at each state and time which the expert draws upon while generating the trajectories in $D$.

\begin{definition}[Expert table] \label{def:policy-table}
The expert table, $\mathbf{T}^*$ is a collection of random variables $\mathbf{T}^*_{t,s} (i)$ indexed by $t \in [H]$, $s \in \mathcal{S}$ and $i = 1,2,\cdots$. Fixing $s \in \mathcal{S}$ and $t \in [H]$, for $i = 1,2,\cdots$, each $\mathbf{T}^\pi_{t,s} (i)$ is drawn independently $\sim \pi_t^* (\cdot | s)$.
\end{definition}

In a sense, the expert table fixes the randomness in the expert's non-determinstic policy. As promised, we next present the alternate view of generating the expert dataset $D$, where the expert sequentially samples actions from the expert table at visited states.

\begin{lemma}[Alternate view of generating $D$] \label[lemma]{lemma:alt-view}
Generate a dataset $D$ of $N$ trajectories as follows: For the $n^{th}$ trajectory $\textsf{tr}_n$, the state $s_1^n$ is drawn independently from $\rho$. The action $a_1^n$ is assigned as the first action from $\mathbf{T}^*_{1,s_1^n} (\cdot)$ that was not chosen in a previous trajectory. Then the MDP independently samples the state $s_2^n \sim P_1 (\cdot|s_1^n,a_1^n)$. In general, at time $t$ the action $a_t^n$ is drawn as the first action in $\mathbf{T}^*_{t,s_t^n} (\cdot)$ that was not chosen at time $t$ in any previous trajectory $n' < n$. The subsequent state $s_{t+1}^n$ is drawn independently $\sim P_{t+1} (\cdot | s_t^n,a_t^n)$.

The probability of generating a dataset $D = \{ \textsf{tr}_1,\cdots,\textsf{tr}_N \}$ by this procedure is $= \prod_{n=1}^N \mathrm{Pr}_{\pi^*} [\textsf{tr}_n]$. This is the same as if the trajectories were generated by independently rolling out $\pi^*$ for $N$ episodes.
\end{lemma}
\begin{proof}
Starting from the initial state $s_1^n \sim \rho$, the probability of $\textsf{tr}_n = \{ (s_1,a_1),\cdots,(s_H,a_H) \}$ is,
\begin{equation*}
    \mathrm{Pr} \Big( \textsf{tr}_n = \{ (s_1,a_1),\cdots,(s_H,a_H)\} \Big) = \rho (s_1) \left( \prod\nolimits_{t=1}^{H-1} \pi_t^* (a_t | s_t) P_t (s_{t+1} | s_t,a_t) \right) \pi_H^* (a_H | s_H).
\end{equation*}
This relies on the fact that each action in $\mathbf{T}^*_{t,s} (\cdot)$ is sampled independently from $\pi^*_t (\cdot | s)$. Carrying out the same calculation for the $n$ trajectories jointly (which we avoid to keep notation simple) results in the claim. The important element remains the same: each action in $\mathbf{T}^*_{t,s_t^n} (\cdot)$ is sampled independently from $\pi^*_t (\cdot | s_t^n)$.
\end{proof}

Note that the process in \Cref{lemma:alt-view} generates a dataset having the same distribution as if the trajectories were generated by independently rolling out $\pi^*$ for $N$ episodes. Without loss of generality we may therefore assume that the expert generates $D$ this way. We adopt this alternate view to enable the coupling between the expert's and learner's policies.

\begin{remark}
We emphasize that the infinite table $\mathbf{T}^*$ is not known to the learner and is only used by the expert to generate the dataset $D$. However, by virtue of observing the trajectories in $D$ the learner is revealed some part of $\mathbf{T}^*$. In particular at the state $s$ and time $t$, the first $|N_{t,s}|$ actions in $\mathbf{T}_{t,s}^*$ are revealed to the learner.
\end{remark}

Recall that $\pi^{\mathrm{first}}_t (\cdot | s)$ defined in \cref{eq:def:pi-first} deterministically plays the action in the first trajectory in $D$ that visits a state $s$ at time $t$, if available, and otherwise plays the uniform distribution $\mathrm{Unif} (\mathcal{A})$.

Using the alternate view of generating $D$ in \Cref{lemma:alt-view}, this policy can be equivalently defined as one which plays the action at the first position in the table $\mathbf{T}^*$ if observed, and otherwise plays the uniform distribution.
\begin{equation}
    \pi^{\mathrm{first}}_t (\cdot | s) =
    \begin{cases}
    \delta_{\mathbf{T}^*_{t,s} (1)} \qquad & \text{if } |N_{t,s}| > 0, \\
    \mathrm{Unif} (\mathcal{A}), & \text{otherwise}.
    \end{cases}
\end{equation}
We now define the oracle policy $\pi^{\mathrm{orc-first}}$, which plays the first action at each time $t \in [H]$ at each state $s \in \mathcal{S}$. That is,
\begin{equation}
    \pi^{\mathrm{orc-first}}_t (\cdot | s) = \delta_{\mathbf{T}^*_{t,s} (1)}.
\end{equation}
With this definition, we first identify that the expected value of $\pi^{\mathrm{orc-first}}$ equals $J(\pi^*)$.

\begin{lemma} \label[lemma]{lemma:expert=orcfirst}
$J (\pi^*) = \mathbb{E} \left[ J (\pi^{\mathrm{orc-first}} ) \right]$.
\end{lemma}

Plugging this into \cref{eq:emp-to-first}, we see that,
\begin{equation} \label{eq:exp-orcfirst}
    J (\pi^*) - \mathbb{E} [J (\widehat{\pi} (D))] = \mathbb{E} \left[ J(\pi^{\mathrm{orc-first}}) - J(\pi^{\mathrm{first}}) \right].
\end{equation}

Observe that $\pi^{\mathrm{orc-first}}$ and $\pi^{\mathrm{first}}$ are in fact identical on all the states that were visited at least once in the expert dataset (i.e. having $|N_{t,s}| > 0$). Therefore, as long as the state $s$ visited at each time $t$ in an episode has $|N_{t,s}| > 0$, both policies collect the same cumulative reward.

\begin{lemma} \label[lemma]{lemma:matchequal}
Fix the expert table $\mathbf{T}^*$ and the expert dataset $D$. Define $\mathcal{E}^c$ as the ``good'' event that the trajectory under consideration only visits a state $s_t$ at each time $t \in [H]$ such that $|N_{t,s_t}| > 0$, i.e. states that have been observed in the expert dataset $D$ at time $t$. Then,
\begin{equation}
    \mathbb{E}_{\pi^{\mathrm{first}}} \left[ \left( \sum\nolimits_{t = 1}^H \mathbf{r}_t (s_t,a_t) \right) \mathbbm{1} \left( \mathcal{E}^c \right) \right] = \mathbb{E}_{\pi^{\mathrm{orc-first}}} \left[ \left( \sum\nolimits_{t = 1}^H \mathbf{r}_t (s_t,a_t) \right) \mathbbm{1} \left( \mathcal{E}^c \right) \right].
\end{equation}
\end{lemma}
\begin{proof}
Both policies are identical on the states such that $|N_{t,s}| > 0$. The event $\mathcal{E}^c$ guarantees that only such states are visited in a trajectory. Therefore both expectations are equal.
\end{proof}

With these preliminaries, we have most of the ingredients to prove the bound on the expected suboptimality of \textsc{Mimic-Emp}. To this end, from \cref{eq:exp-orcfirst} we see that,
\begin{equation}
    J (\pi^*) - J (\widehat{\pi} (D)) = \mathbb{E}_{\pi^{\mathrm{orc-first}}} \left[ \sum\nolimits_{t = 1}^H \mathbf{r}_t (s_t,a_t) \right] - \mathbb{E}_{\pi^{\mathrm{first}}} \left[ \sum\nolimits_{t = 1}^H \mathbf{r}_t (s_t,a_t) \right].
\end{equation}
Subsequently invoking \Cref{lemma:matchequal}, we see that
\begin{align}
    J (\pi^*) - J (\widehat{\pi} (D)) &= \mathbb{E}_{\pi^{\mathrm{orc-first}}} \left[ \left( \sum\nolimits_{t = 1}^H \mathbf{r}_t (s_t,a_t) \right) \mathbbm{1} \left( \mathcal{E} \right) \right] - \mathbb{E}_{\pi^{\mathrm{first}}} \left[ \left( \sum\nolimits_{t = 1}^H \mathbf{r}_t (s_t,a_t) \right) \mathbbm{1} \left( \mathcal{E} \right) \right], \nonumber\\
    &\le \mathbb{E}_{\pi^{\mathrm{orc-first}}} \left[ \left( \sum\nolimits_{t = 1}^H \mathbf{r}_t (s_t,a_t) \right) \mathbbm{1} \left( \mathcal{E} \right) \right], \\
    &\le H \mathrm{Pr}_{\pi^{\mathrm{orc-first}}} \left[ \mathcal{E} \right],
\end{align}
where in the last inequality we use the fact that pointwise $0 \le \mathbf{r}_t \le 1$ for all $t \in [H]$. Taking expectation gives the inequality,
\begin{equation}
    J (\pi^*) - \mathbb{E} [J (\widehat{\pi} (D))] \le H \mathbb{E} [\mathrm{Pr}_{\pi^{\mathrm{orc-first}}} \left[ \mathcal{E} \right]].
\end{equation}
In \Cref{lemma:nondet:error-prob-small} we show that $\mathbb{E} [\mathrm{Pr}_{\pi^{\mathrm{orc-first}}} \left[ \mathcal{E} \right]]$ is upper bounded by $|\mathcal{S}| H \ln (N) / N$, which completes the proof.

\begin{lemma} \label[lemma]{lemma:nondet:error-prob-small}
$\mathbb{E} \left[ \mathrm{Pr}_{\pi^{\mathrm{orc-first}}} [ \mathcal{E} ] \right] \le \frac{|\mathcal{S}| H \ln (N)}{N}$.
\end{lemma}

Although the oracle policy $\pi^{\mathrm{orc-first}}$ and the dataset $D$ are coupled, the key intuition behind showing that the event $\mathcal{E}$ occurs with low probability is that: it is not possible that, in expectation $\pi^{\mathrm{orc-first}}$ visits some state $s$ with high probability, but the same state $s$ visited in the dataset $D$ with low probability. This is by virtue of the fact that in expectation $\pi^{\mathrm{orc-first}}$ matches $\pi^*$ which is the policy that generates $D$.

\subsection{Known-transition setting under deterministic expert policy}

In this section, we describe the proof of \Cref{theorem:det:NsimInf:UB.p1} which upper bounds the expected suboptimality of \textsc{Mimic-MD} (\Cref{alg:det:NsimInf}).

Recall that \textsc{Mimic-MD}, true to its name, mimics the expert on the states observed in half the dataset $D_1$. By virtue of the learner mimicking the expert on states visited in $D_1$, we show that the learner is suboptimal only upon visiting a state unobserved in $D_1$ at some point in an episode.

\begin{lemma} \label[lemma]{lemma:no-error-in-dataset}
Define $\mathcal{E}_{D_1}^{\le t} = \left\{ \exists \tau < t :\ s_t \not\in \mathcal{S}_t (D_1) \right\}$ as the event that the policy under consideration visits some state at time $t$ that no trajectory in $D_1$ has visited at time $t$. Fixing the expert datset $D$, for any policy $\widehat{\pi} \in \Pi_{\mathrm{mimic}} (D_1)$,
\begin{equation}
    J(\pi^*) - J (\widehat{\pi} (D)) = \sum\nolimits_{t=1}^H \Big\{ \mathbb{E}_{\pi^*} \left[ \mathbbm{1} (\mathcal{E}_{D_1}^{\le t}) \mathbf{r}_t (s_t,a_t) \right] - \mathbb{E}_{\widehat{\pi} (D)} \left[ \mathbbm{1} (\mathcal{E}_{D_1}^{\le t}) \mathbf{r}_t (s_t,a_t) \right] \Big\}.
\end{equation}
\end{lemma}
Simplifying this result further using the fact that the reward function is bounded in $[0,1]$ results in \cref{eq:NsimInf:motiv}, recall which we used as a basis for motivating the design of \textsc{Mimic-MD} in \Cref{section:known-transition}. In particular, any policy $\widehat{\pi}$ that exactly mimics the expert on states observed in $D_1$ has suboptimality bounded by,
\begin{equation} \nonumber
    J(\pi^*) - J (\widehat{\pi}) \le \sum_{s \in \mathcal{S}} \sum_{a \in \mathcal{A}} \sum\nolimits_{t=1}^H \left| \mathrm{Pr}_{\pi^*} \Big[ \mathcal{E}_{D_1}^{\le t} ,s_t=s,a_t=a \Big] - \mathrm{Pr}_{\widehat{\pi}} \Big[ \mathcal{E}_{D_1}^{\le t}, s_t=s,a_t=a \Big] \right|.
\end{equation}
The minimum distance functional considered in \textsc{Mimic-MD} simply replaces the population term $\mathrm{Pr}_{\pi^*} [ \mathcal{E}_{D_1}^{\le t} ,s_t=s,a_t=a ]$ by its empirical estimate computed using the dataset $D_2$. We follow the standard analysis of minimum distance function estimators using the triangle inequality, which in effect reduces the analysis to a question of convergence of the empirical estimate of $\mathrm{Pr}_{\pi^*} [ \mathcal{E}_{D_1}^{\le t} ,s_t=\cdot,a_t=\cdot ]$ to the population in $\ell_1$ distance.

Before stating the formal lemma, recall that
\begin{equation}
    \mathcal{T}^{D_1}_t ( s,a) \triangleq \Big\{ \{ (s_1,a_1),\cdots,(s_H,a_H) \} \Big| s_t=s, a_t=a,\ \exists \tau < t : s_\tau \not\in \mathcal{S}_\tau (D_1) \Big\}.
\end{equation}
is defined as the set of trajectories that (i) visits the state $s$ at time $t$, (ii) plays the action $a$ at this time, and (iii) at some time $\tau \le t$ visits a state unobserved in $D_1$.

\begin{lemma} \label[lemma]{lemma:reduc}
Consider any policy $\widehat{\pi}^\varepsilon$ which solves the optimization problem in \ref{eq:opt} to an additive error of $\varepsilon$. Fixing the expert dataset $D$,
\begin{align}\nonumber
    J(\pi^*) - J (\widehat{\pi}^\varepsilon (D)) \le 2\sum_{s \in \mathcal{S}} \sum_{a \in \mathcal{A}} \sum_{t=1}^H \left| \mathrm{Pr}_{\pi^*} \Big[ \mathcal{E}_{D_1}^{\le t}, s_t=s,a_t=a \Big] - \frac{\sum_{\textsf{tr} \in D_2} \mathbbm{1} ( \textsf{tr} \in \mathcal{T}^{D_1}_t ( s,a) )}{|D_2|} \right| + \varepsilon.
\end{align}
\end{lemma}
We emphasize here that $\frac{\sum_{\textsf{tr} \in D_2} \mathbbm{1} ( \textsf{tr} \in \mathcal{T}^{D_1}_t ( s,a) )}{|D_2|}$ is the empirical estimate of $\mathrm{Pr}_{\widehat{\pi}} \Big[ \mathcal{E}_{D_1}^{\le t}, s_t=s,a_t=a \Big]$ computed using the trajectories in the dataset $D_2$.

\begin{remark} \label[remark]{remark:A2}
Taking $\varepsilon = 0$ in \Cref{lemma:reduc} captures the case where $\widehat{\pi}^\varepsilon$ is the policy returned by \textsc{Mimic-MD}.
\end{remark}

The last remaining ingredient in proving the guarantee on the expected suboptimality of \textsc{Mimic-MD} in \Cref{theorem:det:NsimInf:UB.p1} is to bound the convergence rate of the expectation of the RHS of \Cref{lemma:reduc}. We carry out this analysis roughly in two parts:
\begin{enumerate}
    \item[(i)] fixing the dataset $D_1$, for each $t \in [H]$ we bound the convergence rate of the empirical distribution estimate (computed using $D_2$) of $\mathrm{Pr}_{\widehat{\pi}} [ \mathcal{E}_{D_1}^{\le t}, s_t=s,a_t=a ]$ to the population in $\ell_1$ distance, and
    \item[(ii)] we show that the resulting bound (which is a function of $D_1$) has small expectation and converges to $0$ quickly.
\end{enumerate}
This establishes the following bound on the expected suboptimality incurred by \textsc{Mimic-MD}.

\begin{lemma} \label[lemma]{lemma:NsimInf:expec}
\begin{equation}
    \sum_{s \in \mathcal{S}} \sum_{a \in \mathcal{\mathcal{A}}} \sum_{t=1}^H \mathbb{E} \left[ \left| \mathrm{Pr}_{\pi^*} \Big[ \mathcal{T}^{D_1}_t ( s,a) \Big] - \frac{\sum_{\textsf{tr} \in D_2} \mathbbm{1} ( \textsf{tr} \in \mathcal{T}^{D_1}_t ( s,a) )}{|D_2|} \right| \right] \le \min \left\{ \sqrt{\frac{8 |\mathcal{S}| H^2}{N}} ,\ \frac{8}{3} \frac{|\mathcal{S}| H^{\frac{3}{2}}}{N} \right\}.
\end{equation}
\end{lemma}
\noindent In conjunction with \Cref{lemma:reduc} this completes the proof of \Cref{theorem:det:NsimInf:UB.p1} (by plugging in $\varepsilon = 0$ and noting \Cref{remark:A2}) and also \Cref{theorem:det:NsimInf:UB.p3}.

To show the high probability guarantee on \textsc{Mimic-MD} in \Cref{theorem:det:NsimInf:UB.p2}, the key approach is similar. However, we instead
\begin{enumerate}
    \item[(i)] fix $D_1$ and use sub-Gaussian concentration \cite{Boucheron2013ConcentrationI} to establish high probability deviation bounds on the empirical estimate of $\mathrm{Pr}_{\widehat{\pi}} [ \mathcal{E}_{D_1}^{\le t}, s_t=s,a_t=a ]$, and
    \item[(ii)] use missing mass concentration (\Cref{theorem:missingmass-conc}) to show that the resulting deviations (which are a function of $D_1$) concentrate.
\end{enumerate}

\begin{lemma} \label[lemma]{lemma:NsimInf:highprob}
Fix $\delta \in (0, \min \{ 1,H/5\})$. Then, with probability $\ge 1 - \delta$,
\begin{align}
    &\sum_{s \in \mathcal{S}} \sum_{a \in \mathcal{\mathcal{A}}} \sum_{t=1}^H \left| \mathrm{Pr}_{\pi^*} \Big[ \mathcal{T}^{D_1}_t ( s,a) \Big] - \frac{\sum_{\textsf{tr} \in D_2} \mathbbm{1} ( \textsf{tr} \in \mathcal{T}^{D_1}_t ( s,a) )}{|D_2|} \right| \nonumber\\
    &\lesssim \frac{|\mathcal{S}| H^{3/2}}{N} \left( 1 + \frac{3\log (2|\mathcal{S}| H/\delta)}{\sqrt{|\mathcal{S}|}} \right)^{1/2} \sqrt{\log (2|\mathcal{S}| H/\delta)}.
\end{align}
\end{lemma}
The high probability guarantee for \textsc{Mimic-MD} follows suit by invoking \Cref{lemma:reduc} with $\varepsilon = 0$ and \Cref{lemma:NsimInf:highprob}.

\subsection{Proof of lower bounds}

In this section we discuss lower bounds on the expected suboptimality of any algorithm in the no-interaction, active and known-transition settings.

\subsubsection{Active and no-interaction settings}

In this section we discuss the proof of the lower bound in \Cref{theorem:LB.p1} which applies in the no-interaction and active settings. We emphasize that the active setting is strictly a generalization of the no-interaction setting: they are no different if the learner queries and plays the expert's action at each time while interacting with the MDP.

Formally, in the active setting, we assume the learner sequentially rolls out policies $\pi_1,\cdots\pi_N$ to generate trajectories $\textsf{tr}_1,\cdots,\textsf{tr}_N$. The learner is aware of the expert's action at each state visited in each trajectory $\textsf{tr}_n$, however may or may not choose to play this action while rolling out $\pi_n$. We assume that the policy $\pi_n$ is learnt causally, and can depend on all the previous information collected by the learner: the trajectories $\textsf{tr}_1,\cdots,\textsf{tr}_{n-1}$, as well as the expert's policy at each state visited in these trajectories.

\paragraph{Notation:} We use $D = \textsf{tr}_1,\cdots,\textsf{tr}_n$ to denote the trajectories collected by the learner by rolling out $\pi_1,\cdots,\pi_N$. In addition the learner exactly knows the expert's policy $\pi^*_t (\cdot |s)$ at all states $s \in \mathcal{S}_t (D)$. We also define $A = \{ \pi^*_t (\cdot |s) : t \in [H], s \in \mathcal{S}_t (D) \}$ as the expert's policy at states visited in $D$, which is also known to the learner by virtue of actively querying the expert.

The expert policy is deterministic in the lower bound instances we construct. Therefore, we define $\Pi_{\mathrm{mimic}} (D,A)$ (similar to $\Pi_{\mathrm{mimic}} (D)$ in \cref{eq:Pi.mimic}) as the family of deterministic policies which mimics the expert on the states visited in $D$. Namely,
\begin{equation} \label{eq:Pi.mimic.DA}
    \Pi_{\mathrm{mimic}} (D,A) \triangleq \Big\{ \pi \in \Pi_{\mathrm{det}} : \forall t \in [H], s \in \mathcal{S}_t (D),\ \pi_t (s) = \pi^A_t (s) \Big\},
\end{equation}
where $\delta_{\pi^A_t (s)}$ is the policy observed by the learner upon actively querying the expert in a trajectory that visits $s$ at time $t$. Informally, $\Pi_{\mathrm{mimic}} (D,A)$ is the family of expert policies which are ``compatible'' with the dataset $(D,A)$ collected by the learner.

Define $\mathbb{M}_{\mathcal{S}, \mathcal{A}, H}$ as the family of MDPs over state space $\mathcal{S}$, action space $\mathcal{A}$ and epsiode length $H$. In order to prove the lower bound on the worst-case expected suboptimality of any learner $\widehat{\pi} (D,A)$, it suffices to lower bound the Bayes expected suboptimality. Namely, it suffices to find a joint distribution $\mathcal{P}$ over MDPs and expert policies supported on $\mathbb{M}_{\mathcal{S} , \mathcal{A}, H} \times \Pi_{\mathrm{det-exp}}$ such that,
\begin{equation}
    \mathbb{E}_{(\pi^*,\mathcal{M}) \sim \mathcal{P}} \Big[ J_\mathcal{M} (\pi^*) - \mathbb{E} \left[ J_\mathcal{M} (\widehat{\pi} (D,A))\right]\Big] \gtrsim \min \left\{ H, \frac{|\mathcal{S}| H^2}{N} \right\}.
\end{equation}

\paragraph{Construction of $\mathcal{P}$:} First we choose the expert's policy uniformly from $\Pi_{\mathrm{det}}$. That is, for each $t \in [H]$ and $s \in \mathcal{S}$, $\pi^*_t (s) \sim \mathrm{Unif} (\mathcal{A})$. Conditioned on $\pi^*$, the distribution over MDPs induced by $\mathcal{P}$ is deterministic and given by the MDP $\mathcal{M} [\pi^*]$ in \cref{fig:det:Nsim0:LB:repeat}. The $\mathcal{M} [\pi^*]$ is defined with respect to a fixed initial distribution over states $\rho = \{ \zeta,\cdots,\zeta, 1 {-}(|\mathcal{S}|{-}2) \zeta, 0\}$ where $\zeta = \frac{1}{N+1}$. In addition, there is a special state $b \in \mathcal{S}$ which we refer to as the ``bad state''. At each state $s \in \mathcal{S} \setminus \{ b \}$, choosing the expert's action renews the state in the initial distribution $\rho$ and dispenses a reward of $1$, while any other choice of action deterministically transitions to the bad state and offers no reward. In addition, the bad state is absorbing and dispenses no reward irrespective of the choice of action. That is,
\begin{equation}
    P_t (\cdot|s,a) =
    \begin{cases}
    \rho, \qquad &s \in \mathcal{S} \setminus \{ b \},\ a = \pi^*_t (s) \\
    \delta_b, &\text{otherwise},
    \end{cases}
\end{equation}
and the reward function of the MDP is given by,
\begin{equation}
    \mathbf{r}_t (s,a) = \begin{cases} 1, \qquad & s \in \mathcal{S} \setminus \{ b \},\ a = \pi^*_t (s)\\
    0, & \text{otherwise}.
    \end{cases}
\end{equation}

We first state a simple consequence of the construction of the MDP instances and $\mathcal{P}$.

\begin{lemma} \label[lemma]{lemma:exp-gets-H-reward}
Consider any policy $\pi^* \in \Pi_{\mathrm{det}}$. Then, the value of $\pi^*$ on the MDP $\mathcal{M} [\pi^*]$ is $H$.
\end{lemma}
\begin{proof}
Playing the expert's action at any state in $\mathcal{S} \setminus \{ b \}$ is the only way to accrue non-zero reward, and in fact accrues a reward of $1$. In addition, note that the expert never visits the bad state $b$ by virtue of the distribution $\rho$ placing no mass on $b$. Therefore, the value of $\pi^*$ on the MDP $\mathcal{M} [\pi^*]$ is $H$.
\end{proof}

The intuition behind the lower bound construction is as follows. Although the learner can actively query the expert, at the states unvisited in the dataset $D$, the learner has no idea about the expert's policy or the transitions induced under different actions. Intuitively it is clear that the learner cannot guess the expert's action with probability $\ge 1/2$ at such states, a statement which we prove by leveraging the Bayesian construction. In turn, the learner is forced to visit the bad state $b$ at the next point in the episode, and then on collects no reward.

Therefore, to bound the expected reward collected by a learner, it suffices to bound the probability that a learner visits a state unvisited in the expert dataset. The remainder of the proof is in showing that in this MDP construction, in expectation any learner visits such states with probability $\epsilon \gtrsim |\mathcal{S}|/N$ at each point in an episode. Moreover, conditioned on the dataset $D$, these events occur independently across time. Thus informally, the expected suboptimality of a learner is lower bounded by,
\begin{equation}
    H \epsilon + (H-1) \epsilon (1-\epsilon) + \cdots + (1-\epsilon)^H \gtrsim \min \{ H, H^2 \epsilon \}.
\end{equation}
where $\epsilon = |\mathcal{S}|/N$.

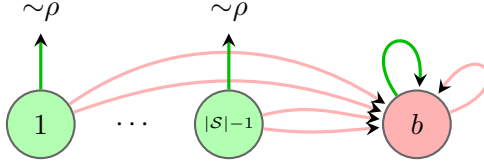
\begin{figure}
\centering
\begin{tikzpicture}[shorten >=1pt,node distance=1.25cm,on grid,auto,good/.style={circle, draw=black!60, fill=green!30!white, thick, minimum size=9mm, inner sep=0pt},bad/.style={circle, draw=black!60, fill=red!30, thick, minimum size=9mm, inner sep=0pt}]
\tikzset{every edge/.style={very thick, draw=red!30!white}}    \tikzset{every loop/.style={min distance=9mm,in=86,out=125, looseness=8, very thick, draw=green!75!black}}
    
\node[good]     (s_1)                   {$1$};
\node           (dist1) [above = 1.5cm of s_1]  {${\sim}\rho$};
\node           (dot)   [right of=s_1]  {$\cdots$};
\node[good]     (s_2)   [right of=dot]  {\tiny $|\mathcal{S}|{-}1$};
\node           (dist2) [above = 1.5cm of s_2]  {${\sim}\rho$};
\node           (s_n)   [right of=s_2]  {};
\node[bad]      (b)     [right of=s_n]  {$b$};

\path[->,>=stealth]
(s_1) edge [draw=green!75!black] node {} (dist1)
      edge [in = 147.5, out = 35]  node {} (b)
      edge [in = 160, out = 20]  node {} (b)
(s_2) edge [draw=green!75!black] node {} (dist2)
      edge [in = 172.5, out = 15]  node {} (b)
      edge [in = 185, out = -7.5]  node {} (b)
(b)   edge [loop above] node {} ()
(b)   edge [distance=10mm,in=55,out=20, very thick] node {} (b);
\end{tikzpicture}
\caption{MDP template when $N_{\mathrm{sim}} = 0$: Upon playing the expert's (green) action at any state except $b$, learner is renewed in the initial distribution $\rho = \{ \zeta,{\cdots},\zeta,1 {-} (|\mathcal{S}|{-}2) \zeta, 0\}$ where $\zeta {=} \frac{1}{N+1}$. Any other choice of action (red) deterministically transitions the state to $b$.}
\label{fig:det:Nsim0:LB:repeat}
\end{figure}

We return to a more formal exposition of the proof of the lower bound. Recall that our objective is to lower bound the Bayes expected suboptimality of $\widehat{\pi}$. Invoking \Cref{lemma:exp-gets-H-reward}, the objective is to lower bound
\begin{equation} \label{eq:Rlowerbound-1}
    \mathbb{E}_{(\pi^*,\mathcal{M}) \sim \mathcal{P}} \Big[ H - \mathbb{E} \Big[ J_{\mathcal{M}} (\widehat{\pi} (D,A))\Big]\Big].
\end{equation}
To this end, we first try to understand the conditional distribution of the expert's policy given the dataset $(D,A)$ collected by the learner. Recall that the dataset $D$ contains trajectories generated by rolling out a sequence of policies $\pi_1,\cdots,\pi_n$, and $A$ captures the expert's policy at states visited in $D$.

\begin{lemma} \label[lemma]{lemma:cond-is-mimic}
Conditioned on the dataset $(D,A)$ collected by the learner, the expert's deterministic policy $\pi^*$ is distributed $\sim \mathrm{Unif} (\Pi_{\mathrm{mimic}} (D,A))$. In other words, at each state visited in the expert dataset, the expert's choice of action is fixed as the one returned when the expert was actively queried at this state. At the remaining states, the expert's choice of action is sampled uniformly from $\mathcal{A}$.
\end{lemma}

\begin{definition} \label{def:PDA}
Define $\mathcal{P} (D,A)$ as the joint distribution of $(\pi^*, \mathcal{M})$ conditioned on the dataset $(D,A)$ collected by the learner. In particular, $\pi^* \sim \mathrm{Unif} (\Pi_{\mathrm{mimic}} (D,A))$ and $\mathcal{M} = \mathcal{M} [\pi^*]$.
\end{definition}

\noindent From \Cref{lemma:cond-is-mimic} and the definition of $\mathcal{P} (D,A)$ in \Cref{def:PDA}, applying Fubini's theorem gives,
\begin{equation} \label{eq:Rlowerbound-2}
    \mathbb{E}_{(\pi^*,\mathcal{M}) \sim \mathcal{P}} \Big[ H - \mathbb{E} \left[ J_{\mathcal{M}} (\widehat{\pi})\right]\Big] = \mathbb{E} \left[ \mathbb{E}_{(\pi^*, \mathcal{M}) \sim \mathcal{P} (D,A)} \left[ H - J_{\mathcal{M}} (\widehat{\pi} (D,A))\right] \right].
\end{equation}
Next we relate this to the first time the learner visits a state unobserved in $D$.

\begin{lemma} \label[lemma]{lemma:hatvalue-UB}
Define the stopping time $\tau$ as the first time $t$ that the learner encounters a state $s_t \ne b$ that has not been visited in $D$ at time $t$. That is,
\begin{equation}
    \tau = \begin{cases}
    \inf \{ t : s_t \not\in \mathcal{S}_t (D) \cup \{ b \} \} \quad & \exists t : s_t \not\in \mathcal{S}_t (D) \cup \{ b \}\\
    H & \text{otherwise}.
    \end{cases}
\end{equation}
Then, conditioned on the dataset $(D,A)$ collected by the learner,
\begin{equation}
    \mathbb{E}_{(\pi^*,\mathcal{M}) \sim \mathcal{P} (D,A)} \Big[ J (\pi^*) - \mathbb{E} \left[ J (\widehat{\pi})\right]\Big] \ge \left( 1 - \frac{1}{|\mathcal{A}|} \right) \mathbb{E}_{(\pi^*,\mathcal{M}) \sim \mathcal{P} (D,A)} \left[ \mathbb{E}_{\widehat{\pi}(D,A)} \left[ H - \tau \right] \right]
\end{equation}
\end{lemma}
Plugging the result of \Cref{lemma:hatvalue-UB} into \cref{eq:Rlowerbound-2}, we have that,
\begin{align}
    \mathbb{E}_{(\pi^*,\mathcal{M}) \sim \mathcal{P}} \Big[ J (\pi^*) - \mathbb{E} \left[ J (\widehat{\pi})\right]\Big] &\ge \left( 1 - \frac{1}{|\mathcal{A}|} \right) \mathbb{E} \left[ \mathbb{E}_{(\pi^*,\mathcal{M}) \sim \mathcal{P} (D,A)} \left[ \mathbb{E}_{\widehat{\pi}} \left[ H - \tau \right] \right] \right], \\
    &\overset{(i)}{\ge} \left( 1 - \frac{1}{|\mathcal{A}|} \right) \frac{H}{2} \mathbb{E} \left[ \mathbb{E}_{(\pi^*,\mathcal{M}) \sim \mathcal{P} (D,A)} \left[ \mathrm{Pr}_{\widehat{\pi}} \Big[ \tau \le \lfloor H/2 \rfloor \Big] \right] \right], \\
    &= \left( 1 - \frac{1}{|\mathcal{A}|} \right) \frac{H}{2} \mathbb{E}_{(\pi^*, \mathcal{M}) \sim \mathcal{P}} \left[ \mathbb{E} \left[ \mathrm{Pr}_{\widehat{\pi}} \Big[ \tau \le \lfloor H/2 \rfloor \Big] \right] \right], \label{eq:Rlowerbound-3}
\end{align}
where $(i)$ uses Markov's inequality and the last equation uses Fubini's theorem.

The last remaining element of he proof is to indeed bound the probability that the learner visits a state unobserved in the dataset before time $\lfloor H/2 \rfloor$.
In \Cref{lemma:failuretime-sumbound} we prove that for any learner $\widehat{\pi}$, $\mathbb{E}_{(\pi^*, \mathcal{M}) \sim \mathcal{P}} \left[ \mathbb{E} \left[ \mathrm{Pr}_{\widehat{\pi}} \left[ \tau \le \lfloor H/2 \rfloor \right] \right] \right]$ is lower bounded by $\gtrsim \min \{ 1, |\mathcal{S}| H / N\}$. Therefore,
\begin{equation}
    \mathbb{E}_{(\pi^*,\mathcal{M}) \sim \mathcal{P}} \Big[ J (\pi^*) - \mathbb{E} \left[ J (\widehat{\pi})\right]\Big] \gtrsim \left( 1 - \frac{1}{|\mathcal{A}|} \right) \frac{H}{2} \min \left\{ 1, \frac{|\mathcal{S}| H}{N} \right\}.
\end{equation}
Since $\left( 1 - \frac{1}{|\mathcal{A}|} \right)$ is a constant for $|\mathcal{A}| \ge 2$ the statement of \Cref{theorem:LB.p1} follows.

\begin{lemma} \label[lemma]{lemma:failuretime-sumbound}
For any learner policy $\widehat{\pi}$,
\begin{equation}
    \mathbb{E}_{(\pi^*,\mathcal{M}) \sim \mathcal{P}} \left[ \mathbb{E} \left[ \mathrm{Pr}_{\widehat{\pi}} \Big[ \tau \le \lfloor H/2 \rfloor \Big] \right] \right] \ge 1 - \left(  1 - \frac{|\mathcal{S}|-2}{e(N+1)} \right)^{\lfloor H/2 \rfloor} \gtrsim \min \left\{ 1, \frac{|\mathcal{S}| H}{N} \right\}.
\end{equation}
\end{lemma}

\subsubsection{Known-transition setting}

As in the proof of \Cref{theorem:LB.p1}, in order to prove the lower bound on the expected suboptimality of any learner $\widehat{\pi} (D,A)$, it suffices lower bound the Bayes expected suboptimality. Namely, it suffices to find a joint distribution $\mathcal{P}$ over MDPs and expert policies supported on $\mathbb{M}_{\mathcal{S} , \mathcal{A}, H} \times \Pi_{\mathrm{det-exp}}$ such that,
\begin{equation}
    \mathbb{E}_{(\pi^*,\mathcal{M}) \sim \mathcal{P}} \Big[ J (\pi^*) - \mathbb{E} \left[ J (\widehat{\pi} (D,P))\right]\Big] \gtrsim \min \left\{ H, \frac{|\mathcal{S}| H}{N} \right\}.
\end{equation}

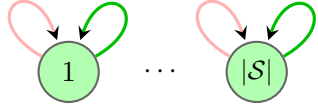
\begin{figure}
\centering
\begin{tikzpicture}[shorten >=1pt,node distance=1.25cm,on grid,auto,good/.style={circle, draw=black!60, fill=green!30!white, thick, minimum size=8mm, inner sep=0pt},bad/.style={circle, draw=black!60, fill=red!30, thick, minimum size=8mm, inner sep=0pt}]
\tikzset{every edge/.style={very thick, draw=red!30!white}}
\tikzset{every loop/.style={min distance=10mm,in=70,out=30,looseness=8, very thick, draw=green!75!black}}
    
\node[good]     (s_1)                   {$1$};
\node           (dot)   [right=of s_1]  {$\cdots$};
\node[good]     (s_n)   [right of=dot]  {\small $|\mathcal{S}|$};

\path[->,>=stealth]
(s_1) edge [distance=10.5mm,in=110,out=150, very thick] node {} (s_1)
(s_1) edge [loop above] node {}     ()
(s_n) edge [distance=10.5mm,in=110,out=150, very thick] node {} (s_n)
(s_n) edge [loop above] node {}     ();
\end{tikzpicture}
\caption{MDP template when $N_{\mathrm{sim}} \to \infty$, Each state is absorbing, initial distribution is given by $\{ \zeta,{\cdots},\zeta, 1 - (|\mathcal{S}|{-}1)\zeta \}$ where $\zeta = \frac{1}{N+1}$}
\label{fig:det:NsimInf:LB:repeat}
\end{figure}

\paragraph{Construction of $\mathcal{P}$} As in the proof of \Cref{theorem:LB.p1}, we first sample the expert's policy uniformly from $\Pi_{\mathrm{det}}$. That is, for each $t \in [H]$ and $s \in \mathcal{S}$, the action $\pi^*_t (s)$ is drawn uniformly from $\mathcal{A}$. Conditioned on $\pi^*$, the distribution over MDPs induced by $\mathcal{P}$ is deterministic and given by the construction $\mathcal{M} [\pi^*]$ in \cref{fig:det:Nsim0:LB:repeat}. $\mathcal{M} [\pi^*]$ is defined with initial distribution over states $\rho = \{ \zeta,\cdots,\zeta, 1 {-}(|\mathcal{S}|{-}1) \zeta\}$ where $\zeta = \frac{1}{N+1}$. Each state $s \in \mathcal{S}$ is absorbing in $\mathcal{M} [\pi^*]$. Formally, for each $s \in \mathcal{S}$ the transition function of $\mathcal{M} [\pi^*]$ is,
\begin{equation}
    P_t (\cdot|s,a) = \delta_s.
\end{equation}
At any state $s$, choosing the expert's action $\pi^*_t (s)$ returns a reward of $1$, while any other choice of action offers $0$ reward.
\begin{equation}
    \mathbf{r}_t (s,a) = \begin{cases} 1, \qquad &a = \pi^*_t (s)\\
    0, & \text{otherwise}.
    \end{cases}
\end{equation}
Note that all the MDPs $\mathcal{M} [\pi^*]$ for $\pi^* \in \Pi_{\mathrm{det}}$ share a common set of transition functions and initial state distribution. Therefore, fixing $P$ and $\rho$, we define $\mathcal{P}'$ to be the joint distribution over expert policies and reward functions induced by $\mathcal{P}$. Then the objective is to lower bound the Bayes expected suboptimality,
\begin{equation} \label{eq:bayes2}
    \mathbb{E}_{(\pi^*,\mathbf{r}) \sim \mathcal{P}'} \Big[ J_{\mathbf{r}} (\pi^*) - \mathbb{E} \left[ J_{\mathbf{r}} (\widehat{\pi} (D,P))\right]\Big].
\end{equation}

In this construction, it is yet again the case that the expert's policy $\pi^*$ collects maximum reward $H$ on $\mathcal{M} [\pi^*]$.

\begin{lemma} \label[lemma]{lemma:exp-gets-H-reward:Inf}
Consider any policy $\pi^* \in \Pi_{\mathrm{det}}$. Then, the value of $\pi^*$ on the MDP $\mathcal{M} [\pi^*]$ is $H$.
\end{lemma}
\begin{proof}
At each state visited $\pi^*$ plays the only action which accrues a reward of $1$. By accumulating a local reward of $1$ at each step, $\pi^*$ has value equal to $H$ on the MDP $\mathcal{M} [\pi^*]$.
\end{proof}

With this explanation, invoking \Cref{lemma:exp-gets-H-reward:Inf} shows that out objective is to now lower bound,
\begin{equation} \label{eq:Rlowerbound-1:Inf}
    \mathbb{E}_{(\pi^*,\mathbf{r}) \sim \mathcal{P}'} \left[ H - \mathbb{E} \left[ J_{\mathbf{r}} (\widehat{\pi} (D,P))\right] \right].
\end{equation}
Similar to \Cref{lemma:cond-is-mimic}, we can compute the conditional distribution of the expert's policy (which marginally follows the uniform prior) given the expert dataset $D$.

\begin{lemma} \label[lemma]{lemma:cond-is-mimic:Inf}
Conditioned on $D$, the distribution of the expert policy $\pi^*$ is uniform over the family of deterministic policies
$\Pi_{\mathrm{mimic}} (D)$ (as defined in \cref{eq:Pi.mimic}).
\end{lemma}

For brevity of notation, we define this conditional distribution of the expert policy given the dataset $D$ by $\mathcal{P}' (D)$.

\begin{definition} \label{def:PD}
Define $\mathcal{P}' (D)$ as the joint distribution of $(\pi^*, \mathbf{r})$ conditioned on the expert dataset $D$. In particular, $\pi^* \sim \mathrm{Unif} (\Pi_{\mathrm{mimic}} (D))$ and $\mathbf{r} = \mathbf{r} [\pi^*]$.
\end{definition}

From \Cref{lemma:cond-is-mimic:Inf} and \Cref{def:PD} and applying Fubini's theorem,
\begin{equation} \label{eq:Rlowerbound-2:Inf}
    \mathbb{E}_{(\pi^*,\mathbf{r}) \sim \mathcal{P}'} \left[ \mathbb{E} \left[ H - J_{\mathbf{r}} (\widehat{\pi} (D,P))\right] \right] = \mathbb{E} \left[ \mathbb{E}_{(\pi^*,\mathbf{r}) \sim \mathcal{P}' (D)} \left[ H - J_{\mathbf{r}} (\widehat{\pi} (D,P))\right] \right].
\end{equation}

Fixing the expert dataset $D$, we subsequently show that the suboptimality of the learner is $\Omega (H)$ if initialized in a state unobserved in the expert dataset $D$. The key intuition is to identify that here the learner's knowledge of the transition function plays no role as each state in the MDP is absorbing. Therefore, once again at states unvisited in the expert dataset, the learner cannot guess the expert's action with high probability at states, leading to errors that grow linearly in $H$.

\begin{lemma} \label[lemma]{lemma:errorbound:Inf}
For any learner's policy $\widehat{\pi}$ conditioned on the expert dataset $D$,
\begin{equation}
    \mathbb{E}_{(\pi^*,\mathbf{r}) \sim \mathcal{P}' (D)} \left[ H - J_{\mathbf{r}} (\widehat{\pi} (D,P))\right] \ge H \left( 1 - \frac{1}{|\mathcal{A}|} \right) \Big( 1 - \rho (\mathcal{S}_1 (D)) \Big).
\end{equation}
\end{lemma}

\noindent Therefore, from \Cref{lemma:errorbound:Inf,eq:Rlowerbound-2:Inf},
\begin{equation} \label{eq:Rlowerbound-3:Inf}
    \mathbb{E}_{(\pi^*,\mathbf{r}) \sim \mathcal{P}'} \left[ \mathbb{E} \left[ H - J_{\mathbf{r}} (\widehat{\pi} (D,P))\right] \right] \ge H \left( 1 - \frac{1}{|\mathcal{A}|} \right) \mathbb{E} \Big[ 1 - \rho (\mathcal{S}_1 (D)) \Big].
\end{equation}
The last ingredient left to show is that the probability mass on states unobserved in the expert dataset, $1 - \rho (\mathcal{S}_1 (D))$, is not too small in expectation. Here we realize that this boils down to calculating the expected missing mass of the distribution $\rho$ given $N$ samples drawn independently. By construction of $\rho$, we show that this is $\gtrsim |\mathcal{S}|/N$ in expectation.

\begin{lemma} \label[lemma]{lemma:NsimInf-lasting}
$\mathbb{E} [ 1 - \rho (\mathcal{S}_1 (D)) ] \ge \frac{|\mathcal{S}|-1}{e(N+1)}$.
\end{lemma}

Plugging \Cref{lemma:NsimInf-lasting} back into \cref{eq:Rlowerbound-3:Inf} certifies a lower bound on the Bayes expected suboptimality of any learner $\widehat{\pi}$. This implies the existence of an MDP on which the learner's expected suboptimality is $\gtrsim |\mathcal{S}| H/ N$.

\newpage

\section{} \label{app:b}

\localtableofcontents

\subsection{Missing proofs for the analysis of behavior cloning}
\subsubsection{Proof of \Cref{lemma:BC-prob-small}}
Since the expert dataset $D$ is composed of trajectories generated by i.i.d. rollouts of $\pi^*$, we have that $\mathrm{Pr} [s \not\in \mathcal{S}_\tau (D)] = ( 1 - \mathrm{Pr}_{\pi^*} [ s_\tau = s ] )^{|D|}$. Therefore,
\begin{align}
    \sum_{t=1}^H \sum_{s \in \mathcal{S}} \mathrm{Pr}_{\pi^*} [s_t = s] \ \mathrm{Pr} [s \not\in \mathcal{S}_t (D)] \le \sum_{\tau = 1}^H \sum_{s \in \mathcal{S}} \mathrm{Pr}_{\pi^*} [ s_\tau = s ] \Big( 1 - \mathrm{Pr}_{\pi^*} [ s_\tau = s ] \Big)^{|D|}. \label{eq:missing-mass-bound}
\end{align}
Noting that $\max_{x \in [0,1]} x (1-x)^N = \frac{1}{N+1} \left( 1 - \frac{1}{N+1} \right)^N \le \frac{4}{9N}$, from \cref{eq:missing-mass-bound},
\begin{equation}
    \sum_{\tau = 1}^H \sum_{s \in \mathcal{S}} \mathrm{Pr}_{\pi^*} [ s_\tau = s ] \Big( 1 - \mathrm{Pr}_{\pi^*} [ s_\tau = s ] \Big)^{|D|} \le \sum_{\tau = 1}^H \sum_{s \in \mathcal{S}} \frac{4}{9 |D|} \le \frac{4}{9} \frac{|\mathcal{S}| H}{|D|}.
\end{equation}

\subsubsection{Proof of \Cref{theorem:missingmass-conc}}

To prove this theorem, we invoke a result of \cite{McAllesterOrtiz} on the concentration of missing mass.

\begin{theorem}[Concentration of missing mass \cite{McAllesterOrtiz}] \label{theorem:generic-conc}
Consider an arbitrary distribution $\nu$ on $\mathcal{X}$, and let $X^N \overset{\text{i.i.d.}}{\sim} \nu$ be a dataset of $N$ samples drawn i.i.d. from $\nu$. Let $\beta \ge 0$ and $\sigma \ge 0$ be constants such that $\sum_{x \in \mathcal{X}} (\nu(x))^2 e^{-(N-\beta) \nu (x)} \le \sigma^2$. For any $0 \le \varepsilon \le \beta \sigma^2$, we have the following,
\begin{equation}
    \Pr \left( \mathfrak{m}_0 (\nu, X^N) - \mathbb{E} [\mathfrak{m}_0 (\nu, X^N)] \ge \varepsilon \right) \le \exp \left(-\frac{\varepsilon^2}{2 \sigma^2} \right).
\end{equation}
\end{theorem}

\noindent We prove \Cref{theorem:missingmass-conc} by an appropriate choice of parameters $\beta,\sigma^2$ and $\epsilon$ (as functions of the confidence parameter $\delta$). In particular, choose $\beta = N - \frac{N}{\sqrt{\log (1/\delta)}} \ge \frac{N}{3}$. For this choice of $\beta$,
\begin{align}
    \sum_{x \in \mathcal{X}} (\nu(x))^2 e^{-(N-\beta) \nu (x)} &= \sum_{x \in \mathcal{X}} (\nu(x))^2 e^{-\frac{N}{\sqrt{\log (1/\delta)}} \nu (x)}, \\
    &\le |\mathcal{X}| \sup_{\nu \in [0,1]} \nu^2 e^{-\frac{N}{\sqrt{\log (1/\delta)}} \nu}, \\
    &\overset{(i)}{=} |\mathcal{X}| \left( 4e^{-2} \frac{\log (1/\delta)}{N^2} \right).
\end{align}
where $(i)$ involves computing the supremum explicitly by differentiation. Therefore, for $\beta = N - \frac{N}{\sqrt{\log (H/\delta)}}$, a feasible choice of $\sigma^2$ in \Cref{theorem:generic-conc} that upper bounds $\sum_{x \in \mathcal{X}} (\nu(x))^2 e^{-(N-\beta) \nu (x)}$ is $\frac{3|\mathcal{X}| \log (1/\delta)}{N^2}$. Choose $\varepsilon = \frac{3\sqrt{|\mathcal{X}|} \log (1/\delta)}{N}$ (note that this choice satisfies $\varepsilon \le \beta \sigma^2$ since $\beta \ge N/3$ and $\sigma^2 = \frac{9|\mathcal{X}| \log (1/\delta)}{N^2}$). Invoking \Cref{theorem:generic-conc} with this choice of $\beta$, $\sigma^2$ and $\epsilon$,
\begin{equation}
    \Pr \left( \mathfrak{m}_0 (\nu, X^N) - \mathbb{E} [\mathfrak{m}_0 (\nu, X^N)] \ge \frac{3\sqrt{|\mathcal{X}|} \log (1/\delta)}{N} \right) \le \exp \left( -\frac{\left( 3\sqrt{|\mathcal{X}|} N^{-1} \log (1/\delta) \right)^2}{9|\mathcal{X}| N^{-2} \log (1/\delta)}\right) = \delta.
\end{equation}
This proves \Cref{theorem:missingmass-conc}.

\subsubsection{Proof of \Cref{lemma:error-prob-conc}}
We decompose $\sum_{\tau=1}^H \sum_{s \in \mathcal{S}} \mathrm{Pr}_{\pi^*} [ s_\tau = s ] \mathbbm{1} (s \not\in \mathcal{S}_\tau (D))$ as $\sum_\tau Z_\tau$ where $Z_\tau = \sum_{s \in \mathcal{S}} \mathrm{Pr}_{\pi^*} [ s_\tau = s ] \mathbbm{1} (s \not\in \mathcal{S}_\tau (D))$.
Observe that for each fixed $\tau$, $Z_\tau$ is in fact the missing mass of the distribution over states at time $\tau$ rolling out $\pi^*$, given $N$ samples from the distribution. Applying the missing mass concentration inequality from \Cref{theorem:missingmass-conc}, with probability $\ge 1 - \delta / H$,
\begin{equation}
    Z_\tau - \mathbb{E} [Z_\tau] \le \frac{3\sqrt{|\mathcal{S}|} \log (H/\delta)}{N}.
\end{equation}
Therefore, by union bounding, with probability $\ge 1 - \delta$,
\begin{equation}
    \sum_{\tau = 1}^H Z_\tau \le \sum_{\tau = 1}^H \mathbb{E} [Z_\tau] + H \cdot \frac{3\sqrt{|\mathcal{S}|} \log (H/\delta)}{N}.
\end{equation}
Using $\sum_{\tau=1}^H Z_\tau = \sum_{\tau = 1}^H \sum_{s \in \mathcal{S}} \mathrm{Pr}_{\pi^*} [ s_\tau = s ] \mathbbm{1} (s \not\in \mathcal{S}_\tau (D))$ and applying \Cref{lemma:BC-prob-small} to claim that $\sum_{\tau=1}^H \mathbb{E} [Z_\tau] \le 4|\mathcal{S}| H / 9N$ completes the proof.

\subsection{Reduction of IL to supervised learning under TV distance (\Cref{lemma:nondet-reduction})}

For each $\tau \in [H]$, define the policy $\widetilde{\pi}^\tau = \{ \pi^*_1,\cdots,\pi^*_\tau,\widehat{\pi}_{\tau+1},\cdots,\widehat{\pi}_H \}$ with $\widetilde{\pi}^0 = \widehat{\pi}$. The policy $\widetilde{\pi}^\tau$ plays the expert's policy till time $\tau$ and the learner's policy for the remainder of the episode. Then,
\begin{align} \label{eq:taubound}
    J(\pi^*) - J (\widehat{\pi}) = \sum\nolimits_{\tau=1}^H J(\widetilde{\pi}^\tau) - J(\widetilde{\pi}^{\tau-1}).
\end{align}
For any fixed $\tau \in [H]$, observe that $\widetilde{\pi}^\tau$ and $\widetilde{\pi}^{\tau-1}$ roll out the same policy till time $\tau-1$. Therefore the expected reward collected until time $\tau-1$ for both policies is the same. By linearity of expectation,
\begin{align} \label{eq:tbound}
    J(\widetilde{\pi}^\tau) - J(\widetilde{\pi}^{\tau-1}) &= \sum\nolimits_{t=\tau}^H \mathbb{E}_{\widetilde{\pi}^\tau} \left[ \mathbf{r}_t (s_t,a_t) \right] - \mathbb{E}_{\widetilde{\pi}^{\tau-1}} \left[ \mathbf{r}_t (s_t,a_t) \right].
\end{align}
Now fix some $t \ge \tau$ and consider $\mathbb{E}_{\widetilde{\pi}^{\tau}} \left[ \mathbf{r}_t (s_t, a_t) \right] - \mathbb{E}_{\widetilde{\pi}^{\tau-1}} \left[ \mathbf{r}_t (s_t, a_t) \right]$. First observe that,
\begin{align}
    \mathbb{E}_{\widetilde{\pi}^{\tau-1}} \left[ \mathbf{r}_t (s_t, a_t) \right] &= \mathbb{E}_{\substack{s_\tau \sim f_{\pi^*}^\tau \\ a_\tau \sim \widehat{\pi}_\tau (\cdot | s_\tau)}} \left[ \ \mathbb{E}_{\widetilde{\pi}^{\tau-1}} \left[ \mathbf{r}_t (s_t, a_t) \middle| s_\tau, a_\tau \right] \ \right], \\
    &= \sum_{s \in \mathcal{S}} \sum_{a \in \mathcal{A}} f_{\pi^*}^\tau (s) \ \widehat{\pi}_\tau (a | s) \ \mathbb{E}_{\widetilde{\pi}^{\tau-1}} \left[ \mathbf{r}_t (s_t, a_t) | s_\tau = s, a_\tau = a \right], \\
    &= \sum_{s \in \mathcal{S}} \sum_{a \in \mathcal{A}} f_{\pi^*}^\tau (s) \ \widehat{\pi}_\tau (a | s) \ \mathbb{E}_{\widehat{\pi}} \left[ \mathbf{r}_t (s_t, a_t) | s_\tau = s, a_\tau = a \right]. \label{eq:breakdown1}
\end{align}
where in the last equation we use the fact that $\widetilde{\pi}^{\tau-1}$ rolls out $\widehat{\pi}$ time $\tau$ onwards, and the fact that we condition on the state visited and action played at time $\tau$. Moreover, we also use the fact that $\mathbf{r}_t (s_t,a_t)$ only depends on $(s_t,a_t)$ which appears at time $t \ge \tau$. Noting that $\widetilde{\pi}^\tau = (\pi^*_1,\cdots,\pi^*_\tau,\widehat{\pi}_{\tau+1},\cdots,\widehat{\pi}_H)$, a similar decomposition gives,
\begin{align}
    \mathbb{E}_{\widetilde{\pi}^\tau} \left[ \mathbf{r}_t (s_t, a_t) \right] &= \sum_{s \in \mathcal{S}} \sum_{a \in \mathcal{A}} f_{\pi^*}^\tau (s) \ \pi^*_\tau (a | s) \ \mathbb{E}_{\widetilde{\pi}^\tau} \left[ \mathbf{r}_t (s_t, a_t) | s_\tau = s, a_\tau = a \right], \\
    &= \sum_{s \in \mathcal{S}} \sum_{a \in \mathcal{A}} f_{\pi^*}^\tau (s) \ \pi^*_\tau (a | s) \ \mathbb{E}_{\widehat{\pi}} \left[ \mathbf{r}_t (s_t, a_t) | s_\tau = s, a_\tau = a \right], \label{eq:breakdown2}
\end{align}
where in the last equation we similarly use the fact that $\widetilde{\pi}^{\tau}$ rolls out $\widehat{\pi}$ time $\tau+1$ onwards, and the fact that we condition on the action played at time $\tau$. Subtracting \cref{eq:breakdown1} from \cref{eq:breakdown2},
\begin{align}
    \mathbb{E}_{\widetilde{\pi}^\tau} &\left[ \mathbf{r}_t (s_t, a_t) \right] - \mathbb{E}_{\widetilde{\pi}^{\tau-1}} \left[ \mathbf{r}_t (s_t, a_t) \right] \nonumber\\
    &\le \sum_{s \in \mathcal{S}} f_{\pi^*}^\tau (s) \sum_{a \in \mathcal{A}} \mathbb{E}_{\widehat{\pi}} \left[ \mathbf{r}_t (s_t, a_t) | s_\tau = s, a_\tau = a \right] \Big( \pi^*_\tau (a | s) - \widehat{\pi}_\tau (a | s) \Big).
\end{align}
Observe that $\mathbb{E}_{\widehat{\pi}} \left[ \mathbf{r}_t (s_t, a_t) | s_\tau = s, a_\tau = a \right]$ is a function of $(s,a)$ and is bounded in $[0,1]$ (since pointwise $0 \le \mathbf{r}_t \le 1$). Therefore,
\begin{align}
    \mathbb{E}_{\widetilde{\pi}^\tau} \left[ \mathbf{r}_t (s_t, a_t) \right] - \mathbb{E}_{\widetilde{\pi}^{\tau-1}} \left[ \mathbf{r}_t (s_t, a_t) \right]
    &\le \sum_{s \in \mathcal{S}} f_{\pi^*}^\tau (s) \sup_{g : \mathcal{A} \to [0,1]} \sum_{a \in \mathcal{A}} g (a) \Big( \pi^*_\tau (a | s) - \widehat{\pi}_\tau (a | s) \Big), \\
    &\overset{(i)}{=} \sum_{s \in \mathcal{S}} f_{\pi^*}^\tau (s) \textsf{TV}  \Big( \pi^*_\tau (a | s) , \widehat{\pi}_\tau (a | s) \Big), \\
    &= \mathbb{E}_{s \sim f_{\pi^*}^\tau} \left[ \textsf{TV}  \Big( \pi^*_\tau (a | s) , \widehat{\pi}_\tau (a | s) \Big) \right].
\end{align}
where $(i)$ uses the dual representation of TV distance. Summing over $t \ge \tau$ and $\tau \in [H]$ and invoking \cref{eq:taubound,eq:tbound} we get,
\begin{equation}
    J(\pi^*) - J(\widehat{\pi}) \le H \sum_{\tau=1}^H \mathbb{E}_{s \sim f_{\pi^*}^\tau} \left[ \textsf{TV}  \Big( \pi^*_\tau (a | s) , \widehat{\pi}_\tau (a | s) \Big) \right].
\end{equation}
Using the definition of $\mathbb{T}_{\mathrm{pop}}$ (\cref{eq:TVrisk}) completes the proof.

\subsection{Missing proofs for \Cref{theorem:nondet:UB}}

\subsubsection{Proof of \Cref{lemma:val-equal}}

Recall that we assume that the trajectories in the expert dataset are ordered arbitrarily as $\{ \textsf{tr}_1,\cdots,\textsf{tr}_N\}$ where $\textsf{tr}_n = \{ (s_1^n,a_1^n),\cdots,(s_H^n,a_H^n) \}$. $N_{t,s} = \{ n \in [N] : s_t^n = s \}$ as defined in \cref{eq:Nts} is the set of indices of trajectories in $D$ that visit the state $s$ at time $t$. In order to prove this result, suppose the learner's policy $\widehat{\pi}$

With this, we define the randomized stochastic policy $X^{\mathrm{unif}} (D)$ as,
\begin{equation} \label{eq:Xdef}
    X^{\mathrm{unif}}_t (\cdot | s) =
    \begin{cases}
    \delta_{a_t^{\!n (t,s)}} \qquad &\text{if } |N_{t,s}| \ge 1,\\
    \mathrm{Unif} (\mathcal{A}) \qquad &\text{otherwise.}
    \end{cases}
\end{equation}
where each $n (t,s)$ is a random variable independently sampled from $\mathrm{Unif} (N_{t,s})$ whenever $N_{t,s} \ne \emptyset$. Note that fixing $D$ and $n(t,s)$ for all $t,s$ such that $N_{t,s} \ne \emptyset$, the random variable $X^{\mathrm{unif}}$ is a fixed stochastic policy.

The policy $X^{\mathrm{unif}} (D)$ in a sense corresponds to just extracting the randomness in the actions chosen at visited states in the policy $\widehat{\pi} (D)$ returned by \textsc{Mimic-Emp}.

In particular, it is a short proof to see that the random variables $J(X^{\mathrm{unif}} (D))$ and $J(\widehat{\pi} (D))$ have the same expectation.

\begin{lemma}
$\mathbb{E} [J(\widehat{\pi} (D))] = \mathbb{E} [J (X^{\mathrm{unif}} (D))]$.
\end{lemma}
\begin{proof}
Consider some trajectory $\textsf{tr} = \{ (s_1,a_1),\cdots,(s_H,a_H) \}$. Fixing the expert dataset $D$,
\begin{equation}
    \mathbb{E} \Big[\mathrm{Pr}_{X^{\mathrm{unif}} (D)} [\textsf{tr}] \Big| D \Big] = \mathbb{E} \left[ \rho(s_1) \left( \prod\nolimits_{t=1}^{H-1} X^{\mathrm{unif}}_t (a_t | s_t) P_t (s_{t+1} | s_t,a_t) \right) X^{\mathrm{unif}}_t (a_H | s_H) \right].
\end{equation}
From \cref{eq:Xdef} and \Cref{alg:nondet:Nsim0}, observe that $X_t^{\mathrm{unif} (\cdot | s)} = \widehat{\pi} (\cdot |s) = \mathrm{Unif} (\mathcal{A})$ at states $s : N_{t,s} = \emptyset$ (i.e. which were not visited in the expert dataset). Moreover, on the remaining states $X^{\mathrm{unif}}_t (a_t | s_t)$ is independently sampled from the empirical distribution over states at time $t$. In particular, this means that $\mathbb{E} [X^{\mathrm{unif}}_t (a_t | s_t)] = \widehat{\pi}_t (a_t | s_t)$. Plugging this in gives,
\begin{equation}
    \mathbb{E} [\mathrm{Pr}_{X^{\mathrm{unif}} (D)} [\textsf{tr}]] = \mathrm{Pr}_{\widehat{\pi} (D)} [\textsf{tr}].
\end{equation}
Multiplying both sides by $\sum_{t=1}^H \mathbf{r}_t (s_t,a_t)$, summing over all trajectories \textsf{tr} and taking expectation with respect to the expert dataset $D$ completes the proof.
\end{proof}

First we provide an auxiliary result that is critical to showing that the policies $J(X^{\mathrm{unif}} (D))$ and $\pi^{\mathrm{first}} (D)$ have the same value in expectation.

To this end, first define $D_{\le \tau, < \tau} = \{ ( (s_1^n,a_1^n),\cdots,(s_{\tau-1}^n,a_{\tau-1}^n), s_\tau^n ) : n \in [N] \}$ to be the truncation of the expert dataset $D$ till time $\tau$, excluding the actions played at this time. $D_{\le \tau, \le \tau}$ and other similar notations are defined analogously.

\begin{lemma} \label[lemma]{lemma:indep}
Condition on $D_{\le \tau, < \tau}$ which represents the truncation of trajectories in the expert dataset $D$ till the state visited at time $\tau$. At any state $s$ that is visited at least once in $D$ at time $\tau$ (namely with $|N_{\tau,s}| > 0$), the actions $\{ a_\tau^n : n \in N_{\tau,s}\}$ played at trajectories that visit the state $s$ at time $\tau$ are drawn independently and identically $\sim \pi^*_\tau (\cdot|s)$.
\end{lemma}
\begin{proof}
Recall that we condition on $D_{\le \tau, < \tau}$ which captures trajectories in the expert dataset truncated till the state visited at time $\tau$. Since each trajectory $\textsf{tr}_n \in [N]$ is rolled out independently, the action $a_\tau^n$ in each trajectory $\textsf{tr}_n$ is drawn independently from $\pi^*_\tau (\cdot | s_\tau^n)$.

More importantly, conditioned on $D_{\le \tau, < \tau}$ the states $s_\tau^n$ visited in different trajectories is determined. This implies that $N_{\tau,s}$ for $s \in \mathcal{S}$ is a measurable function of $D_{\le \tau, < \tau}$.

These two statements together imply that states $s \in \mathcal{S}$ having $N_{\tau,s} > 0$ (which is a measurable function of $D_{\le \tau, < \tau}$) are such that all the actions $\{ a_\tau^n : n \in N_{\tau,s} \}$ are independent.
\end{proof}

\begin{proof}[Proof of \Cref{lemma:val-equal}]
In order to prove this result, we use an inductive argument. The induction hypothesis is that the expected value of $X^{\mathrm{unif}} (D)$ and $\pi^{\mathrm{first}} (D)$ are the same, conditioned on the expert dataset till time $t$ and the actions from the empirical distribution sampled by $X^{\mathrm{unif}} (D)$ at different states till time $t$. We formalize this hypothesis in equations after first proving the base case. To recognize the fact that we prove the statement starting from $t=H$, we define $\mathcal{H}_H$ as the base case, and inductively prove $\mathcal{H}_{t-1}$ assuming the hypothesis $\mathcal{H}_t$.

First observe that,
\begin{equation}
    \mathbb{E} \Big[ J(X^{\mathrm{unif}} (D)) \Big| D_{\le H,<H}, \Big\{ n (t,s) \ \Big| \ t \le H,\ s : N_{t,s} > 0 \Big\} \Big] = \mathbb{E} \Big[ J(\pi^{\mathrm{first}} (D)) \Big| D_{\le H,< H} \Big].
\end{equation}
This is because conditioned on $D_{\le H,<H}$, the only randomness is in the actions that are played in the different trajectories at time $H$. By \Cref{lemma:indep} these are distributed i.i.d. $\sim \pi_t^* (\cdot | s)$. Taking expectation with respect to $\left\{ n_{H,s} \middle| s : N_{t,s} > 0 \right\}$, results in proof of the base case for $t = H$,
\begin{equation} \nonumber
    \mathcal{H}_H : \mathbb{E} \Big[ J(X^{\mathrm{unif}} (D)) \Big| D_{\le H,<H}, \Big\{ n (t,s) \ \Big| \ t < H,\ s : N_{t,s} > 0 \Big\} \Big] = \mathbb{E} \Big[ J(\pi^{\mathrm{first}} (D)) \Big| D_{\le H,< H} \Big].
\end{equation}
In general consider the hypothesis $\mathcal{H}_\tau$,
\begin{equation} \nonumber
    \mathcal{H}_\tau : \mathbb{E} \Big[ J(X^{\mathrm{unif}} (D)) \Big| D_{\le \tau, <\tau}, \Big\{ n (t,s) \ \Big| \ t < \tau,\ s : N_{t,s} > 0 \Big\} \Big] = \mathbb{E} \Big[ J(\pi^{\mathrm{first}} (D)) \Big| D_{\le \tau, <\tau} \Big].
\end{equation}
Taking expectation with respect to $\{ s_\tau^n : n \in [N] \}$, where conditionally $s_\tau^n \sim P_\tau (\cdot | s_{\tau-1}^n, a_{\tau-1}^n)$,
\begin{equation}
    \mathbb{E} \Big[ J(X^{\mathrm{unif}} (D)) \Big| D_{< \tau,<\tau}, \Big\{ n (t,s) \ \Big| \ t < \tau,\ s : N_{t,s} > 0 \Big\} \Big] = \mathbb{E} \Big[ J(\pi^{\mathrm{first}} (D)) \Big| D_{< \tau,<\tau} \Big].
\end{equation}
Next we take expectation with respect to the actions $\{ a_\tau^n : n \in [N] \}$ where each $a_\tau^n$ is drawn independently from $\pi^*_t (\cdot | s_\tau^n)$. This results in,
\begin{equation}
    \mathbb{E} \Big[ J(X^{\mathrm{unif}} (D)) \Big| D_{< \tau, < \tau-1}, \Big\{ n (t,s) \ \Big| \ t < \tau,\ s : N_{t,s} > 0 \Big\} \Big] = \mathbb{E} \Big[ J(\pi^{\mathrm{first}} (D)) \Big| D_{< \tau, < \tau-1} \Big].
\end{equation}
Note that on both sides we condition on $D_{< \tau, < \tau-1}$ which is the set of partial trajectories in the expert dataset till time $\tau-1$ (excluding the action at this time). In particular, this conditioning determines the set of states visited at time $\tau-1$ in the expert dataset. Consider any state $s \in \mathcal{S}$:
\begin{enumerate}
    \item[(i)] If $s$ was not observed in the dataset $D$ at time $\tau-1$, then with probability $1$ over the randomness of $X^{\mathrm{unif}}$, both the policies $X^{\mathrm{unif}}$ and $\pi^{\mathrm{first}}$ play the policy $\mathrm{Unif} (\mathcal{A})$;
    \item[(ii)] On the other hand, if $s$ was observed in the dataset $D$ in some trajectory at time $\tau-1$, then $X^{\mathrm{unif}}$ samples from an empirical distribution over actions played at the state $s$ in the dataset at time $\tau-1$, which by \Cref{lemma:indep} are drawn independently from $\pi^*_{\tau-1} (\cdot | s)$. On the other hand, the action played by $\pi^{\mathrm{first}}$ is also drawn independently from $\pi^*_{\tau-1} (\cdot | s)$. This shows that the expectation on the LHS does not depend on the choice of $n(\tau-1,s)$ for any state $s \in \mathcal{S}$.
\end{enumerate}
Thus in both cases, the expectation of the random variable on the RHS does not depend on $\{ n(\tau-1,s) | s \in \mathcal{S} \}$. Therefore, we can drop the conditioning on this random variable to give,
\begin{equation}
    \mathbb{E} \Big[ J(X^{\mathrm{unif}} (D)) \Big| D_{< \tau, < \tau-1}, \Big\{ n (t,s) \ \Big| \ t < \tau-1,\ s : N_{t,s} > 0 \Big\} \Big] = \mathbb{E} \Big[ J(\pi^{\mathrm{first}} (D)) \Big| D_{< \tau, < \tau-1} \Big]
\end{equation}
This proves the induction hypothesis $\mathcal{H}_{\tau-1}$ and consequently the hypothesis $\mathcal{H}_1$. Taking expectation on both sides of $\mathcal{H}_1$ with respect to $s_1^n \overset{\text{i.i.d.}}{\sim} \rho$ proves the claim.
\end{proof}

\subsubsection{Proof of \Cref{lemma:expert=orcfirst}}

Fixing the table $\mathbf{T}^*$, the probability of observing the trajectory $\textsf{tr} = \{ (s_1,a_1),\cdots,(s_H,a_H) \}$ under the deterministic policy $\pi^{\mathrm{orc-first}}$ is,
\begin{equation}
    \mathrm{Pr}_{\pi^{\mathrm{orc-first}}} (\textsf{tr}) = \rho (s_1) \left( \prod_{t=1}^{H-1} \mathbbm{1} \left(a_t = \mathbf{T}^*_{t,s_t} (1) \right) P_t (s_{t+1} | s_t,a_t) \right) \mathbbm{1} \left( a_H = \mathbf{T}^*_{H,s_H} (1) \right).
\end{equation}
Since the actions $\mathbf{T}^*_{t,s_t} (1)$ are independently drawn from $\pi^*_t (\cdot | s_t)$, taking expectation, we see that
\begin{equation} \label{eq:oracle=expert}
    \mathbb{E} \left[ \mathrm{Pr}_{\pi^{\mathrm{orc-first}}} (\textsf{tr}) \right] = \rho (s_1) \left( \prod_{t=1}^{H-1} \pi_t^* (a_t | s_t) P_t (s_{t+1} | s_t,a_t) \right) \pi_H^* (a_H | s_H) = \mathrm{Pr}_{\pi^*} (\textsf{tr}).
\end{equation}
Multiplying both sides by $\sum_{t=1}^H \mathbf{r}_t (s_t,a_t)$ and summing over all trajectories completes the proof.

\subsubsection{Proof of \Cref{lemma:nondet:error-prob-small}}

Recall that the ``failure'' $\mathcal{E}$ is defined as the event that at some time $t \in [H]$, a state $s_t$ is visited such that $|N_{t,s_t}| = 0$, i.e. that was not visited in the expert dataset. By union bounding,
\begin{equation} \label{eq:union-bound}
    \mathbb{E} \left[ \mathrm{Pr}_{\pi^{\mathrm{orc-first}}} [ \mathcal{E} ] \right] \le \sum_{t = 1}^H \sum_{s \in \mathcal{S}} \mathbb{E} \left[ \mathrm{Pr}_{\pi^{\mathrm{orc-first}}} [ \mathcal{E}_{s,t} ] \right],
\end{equation}
where $\mathcal{E}_{s,t}$ is the event that a failure occurs at the state $s$ at time $t$, i.e. the state $s$ is visited at time $t$ and $|N_{t,s}| = 0$. $\mathcal{E}_{s,t}$ is the intersection of two events. Therefore we have the upper bound,
\begin{align} \label{eq:couple-bound}
    \mathbb{E} \left[ \mathrm{Pr}_{\pi^{\mathrm{orc-first}}} [ \mathcal{E}_{s,t} ] \right] &\le \min \Big\{ \mathbb{E} \left[ \mathrm{Pr}_{\pi^{\mathrm{orc-first}}} [ s_t = s ] \right], \mathbb{E} \left[ \mathrm{Pr}_{\pi^{\mathrm{orc-first}}} \left[ |N_{s,t}| = 0 \right] \right] \Big\}.
\end{align}
Observe that these two terms in the minimum are easy to compute. Firstly, using \cref{eq:oracle=expert}, we have that,
\begin{equation} \label{eq:min1}
    \mathbb{E} \left[ \mathrm{Pr}_{\pi^{\mathrm{orc-first}}} [ s_t = s ] \right] = \mathrm{Pr}_{\pi^*} [s_t = s].
\end{equation}
On the other hand,
\begin{equation} \label{eq:min2}
    \mathbb{E} \left[ \mathrm{Pr}_{\pi^{\mathrm{orc-first}}} \left[ |N_{s,t}| = 0 \right] \right] = \mathbb{E} [\mathbbm{1} (|N_{s,t}| = 0)] = (1 - \mathrm{Pr}_{\pi^*} [s_t = s])^N
\end{equation}
where the last equation uses \Cref{lemma:alt-view}. Putting together \cref{eq:min1,eq:min2} with \cref{eq:couple-bound},
\begin{equation}
    \mathbb{E} \left[ \mathrm{Pr}_{\pi^{\mathrm{orc-first}}} [ \mathcal{E}_{s,t} ] \right] \le \min \left\{ \mathrm{Pr}_{\pi^*} [s_t = s],\ \Big(1 - \mathrm{Pr}_{\pi^*} [s_t = s] \Big)^N \right\}.
\end{equation}
In \Cref{lemma:min2} we show that the RHS is upper bounded by $\log (N)/N$. Therefore,
\begin{equation}
    \mathbb{E} \left[ \mathrm{Pr}_{\pi^{\mathrm{orc-first}}} [ \mathcal{E}_{s,t} ] \right] \le \frac{\log N}{N}.
\end{equation}
Plugging back into \cref{eq:union-bound} completes the proof.

\begin{lemma} \label[lemma]{lemma:min2}
For any $x \in [0,1]$ and $N > 1$, $\min \{ x , (1 - x)^N \} \le \frac{\log N}{N}$.
\end{lemma}
\begin{proof}
$x$ is an increasing function, while $(1 - x)^N$ is decreasing. For $x = \frac{\log N}{N}$,
\begin{equation}
    (1 - x)^N = \left( 1 - \frac{\log N}{N} \right)^N \le e^{- \log N} \le N^{-1}
\end{equation}
Therefore for $x \ge \frac{\log (N)}{N}$, $\min \{ x , (1 - x)^N \} \le \frac{1}{N}$. Therefore $\min \{ x , (1 - x)^N \}  \le \frac{\log N}{N}$.
\end{proof}

\subsection{Missing proofs for \Cref{theorem:det:NsimInf:UB.p1,theorem:det:NsimInf:UB.p2}}

\subsubsection{Proof of \Cref{lemma:no-error-in-dataset}}

Observe that the complement $\left( \mathcal{E}_{D_1}^{\le t} \right)^c$ is the event that the policy under consideration until (and including) time $t-1$, only visits states that were visited in at least one trajectory in the expert dataset. 

\noindent First observe that, fixing the expert dataset $D$,
\begin{align}
    &J(\pi^*) - J(\widehat{\pi} (D)) \\
    &= \mathbb{E}_{\pi^*} \left[ \sum\nolimits_{t=1}^H  \mathbf{r}_t (s_t,a_t) \right] - \mathbb{E}_{\widehat{\pi}} \left[ \sum\nolimits_{t=1}^H \mathbf{r}_t (s_t,a_t) \right] \\
    &= \sum\nolimits_{t=1}^H \mathbb{E}_{\pi^*} \left[ \left(\mathbbm{1} \left( \left(\mathcal{E}^{\le t}_{D_1} \right)^c \right) + \mathbbm{1} \left(\mathcal{E}^{\le t}_{D_1} \right) \right) \mathbf{r}_t (s_t,a_t) \right] - \mathbb{E}_{\widehat{\pi}} \left[ \left(\mathbbm{1} \left( \left(\mathcal{E}^{\le t}_{D_1} \right)^c \right) + \mathbbm{1} \left(\mathcal{E}^{\le t}_{D_1} \right) \right)\mathbf{r}_t (s_t,a_t) \right].
\end{align}
Indeed, to prove the statement it suffices to prove that,
\begin{equation} \label{eq:good-indic-reward}
    \sum\nolimits_{t=1}^H \mathbb{E}_{\pi^*} \left[ \mathbbm{1} \left( \left(\mathcal{E}^{\le t}_{D_1} \right)^c \right) \mathbf{r}_t (s_t,a_t) \right] = \sum\nolimits_{t=1}^H \mathbb{E}_{\widehat{\pi}} \left[ \mathbbm{1} \left( \left( \mathcal{E}_{D_1}^{\le t} \right)^c \right) \mathbf{r}_t (s_t,a_t) \right].
\end{equation}
Recall that the learner $\widehat{\pi}$ mimics the expert at all the states observed in the dataset $D_1$, i.e. having $|N_{t,s}| > 0$. Observe that when the event $\left( \mathcal{E}_{D_1}^{\le t} \right)^c$ occurs, all the states visited in a trajectory have $|N_{t,s}| > 0$. Thus, both expectations are carried out with respect to the same policy and are hence equal. More precisely, for any $t \in [H]$,
\begin{align}
    \mathbb{E}_{\widehat{\pi}} \left[ \mathbbm{1} \left( \left( \mathcal{E}_{D_1}^{\le t} \right)^c \right) \mathbf{r}_t (s_t,a_t) \right] &= \mathbb{E}_{s_1 \sim \rho,\ \tau \le t, \substack{a_\tau \sim \widehat{\pi}_\tau (\cdot |s_\tau) \\ s_{\tau+1} \sim P(\cdot | s_\tau,a_\tau)}} \left[ \mathbbm{1} \left( \left( \mathcal{E}_{D_1}^{\le t} \right)^c \right) \mathbf{r}_t (s_t,a_t) \right] \\
    &\overset{(i)}{=} \mathbb{E}_{s_1 \sim \rho,\ \tau \le t, \substack{a_\tau \sim \pi_\tau^* (\cdot | s_\tau) \\ s_{\tau+1} \sim P(\cdot | s_\tau, a_\tau)}} \left[ \mathbbm{1} \left( \left( \mathcal{E}_{D_1}^{\le t} \right)^c \right) \mathbf{r}_t (s_t,a_t)\right] \\
    &= \mathbb{E}_{\pi^*} \left[ \mathbbm{1} \left( \left( \mathcal{E}_{D_1}^{\le t} \right)^c \right) \mathbf{r}_t (s_t,a_t) \right]
\end{align}
where $(i)$ uses the fact that when $s_\tau \in \mathcal{S}_t (D_1)$ (as implied by $\left(\mathcal{E}_{D_1}^{\le t} \right)^c$ for each $\tau \le t$), then, $\pi^*_\tau (\cdot | s_\tau) = \widehat{\pi}_\tau (\cdot | s_\tau)$. Moreover.

\subsubsection{Proof of \Cref{lemma:reduc}}
First observe that we can write the reward $\mathbf{r}_t (s_t,a_t)$ accrued in some trajectory at time $t$ equals $\sum_{s \in \mathcal{S}} \sum_{a \in \mathcal{A}} \mathbf{r}_t (s,a) \mathbbm{1} ((s_t,a_t) = (s,a))$. Therefore, from \Cref{lemma:no-error-in-dataset},
\begin{align}
    J(\pi^*) - J (\widehat{\pi}^\varepsilon) &= \sum_{s \in \mathcal{S}} \sum_{a \in \mathcal{A}} \sum_{t=1}^H \mathbf{r}_t (s,a) \left( \mathrm{Pr}_{\pi^*} \Big[ \mathcal{E}_{D_1}^{\le t},s_t{=}s,a_t{=}a \Big] - \mathrm{Pr}_{\widehat{\pi}} \Big[ \mathcal{E}_{D_1}^{\le t}, s_t{=}s,a_t{=}a \Big] \right) \nonumber\\
    &\le \sum_{s \in \mathcal{S}} \sum_{a \in \mathcal{A}} \sum_{t=1}^H \left| \mathrm{Pr}_{\pi^*} \Big[ \mathcal{E}_{D_1}^{\le t},s_t=s,a_t=a \Big] - \mathrm{Pr}_{\widehat{\pi}} \Big[ \mathcal{E}_{D_1}^{\le t}, s_t=s,a_t=a \Big] \right| \nonumber\\
    &= \sum_{s \in \mathcal{S}} \sum_{a \in \mathcal{A}} \sum_{t=1}^H \left| \mathrm{Pr}_{\pi^*} \Big[ \mathcal{T}^{D_1}_t ( s,a) \Big] - \mathrm{Pr}_{\widehat{\pi}} \Big[ \mathcal{T}^{D_1}_t ( s,a) \Big] \right|
\end{align}
where the inequality follows from the assumption that $0 \le \mathbf{r}_t (s,a) \le 1$ and the last equation follows from the definition $\mathcal{T}^{D_1}_t ( s,a) = \{ \{ (s_\tau,a_\tau) \}_{\tau=1}^H | s_t{=}s, a_t{=}a,\ \exists \tau {\in} [H] : s_\tau \not\in \mathcal{S}_\tau (D_1) \}$ is the set of trajectories that visit $(s,a)$ at time $t$ and at some point $t'$ in the episode visit a state not visited in any trajectory at time $t'$ in $D_1$. Using the definition of the learner's policy $\widehat{\pi}$ in the optimization problem \eqref{eq:opt} and applying the triangle inequality,
\begin{align}
    J(\pi^*) - J (\widehat{\pi}^\varepsilon) \le \sum_{s \in \mathcal{S}} \sum_{a \in \mathcal{A}} \sum_{t=1}^H &\left| \mathrm{Pr}_{\pi^*} \Big[ \mathcal{T}^{D_1}_t ( s,a) \Big] - \frac{\sum_{\textsf{tr} \in D_2} \mathbbm{1} ( \textsf{tr} \in \mathcal{T}^{D_1}_t ( s,a) )}{|D_2|} \right| \nonumber\\
    + &\left| \frac{\sum_{\textsf{tr} \in D_2} \mathbbm{1} (\textsf{tr} \in \mathcal{T}^{D_1}_t ( s,a) )}{|D_2|} - \mathrm{Pr}_{\widehat{\pi}} \Big[ \mathcal{T}^{D_1}_t ( s,a) \Big] \right|.
\end{align}
Observe that the expert's policy $\pi^*$ is a feasible policy to the optimization problem \eqref{eq:opt}. Since $\widehat{\pi}$ solves \eqref{eq:opt} up to an additive error of $\varepsilon$, we have the upper bound,
\begin{align}
    J(\pi^*) - J (\widehat{\pi}^\varepsilon) &\le 2\sum_{s \in \mathcal{S}} \sum_{a \in \mathcal{A}} \sum_{t=1}^H \left| \mathrm{Pr}_{\pi^*} \Big[ \mathcal{T}^{D_1}_t ( s,a) \Big] - \frac{\sum_{\textsf{tr} \in D_2} \mathbbm{1} ( \textsf{tr} \in \mathcal{T}^{D_1}_t ( s,a) )}{|D_2|} \right| + \varepsilon.
\end{align}

\subsubsection{Proof of \Cref{lemma:NsimInf:expec}}
Recall that we carry out sample splitting in \Cref{alg:det:NsimInf} to give datasets $D_1$ and $D_2$. We first fix the trajectories in $D_1$ and compute the expectation with respect to the dataset $D_2$. Sample splitting implies that, conditioned on $D_1$, the trajectories in $D_2$ are still generated by independently rolling out $\pi^*$. By Jensen's inequality, we can upper bound by the quadratic deviation,
\begin{align}
    &\sum_{s \in \mathcal{S}} \sum_{a \in \mathcal{A}} \sum_{t=1}^H \mathbb{E} \left[ \left| \mathrm{Pr}_{\pi^*} \Big[ \mathcal{T}^{D_1}_t ( s,a) \Big] - \frac{\sum_{\textsf{tr} \in D_2} \mathbbm{1} ( \textsf{tr} \in \mathcal{T}^{D_1}_t ( s,a) )}{|D_2|} \right| \right] \nonumber\\
    &\quad \le \sum_{s \in \mathcal{S}} \sum_{a \in \mathcal{A}} \sum_{t=1}^H \left( \mathbb{E} \left[ \left( \mathrm{Pr}_{\pi^*} \Big[ \mathcal{T}^{D_1}_t ( s,a) \Big] - \frac{\sum_{\textsf{tr} \in D_2} \mathbbm{1} ( \textsf{tr} \in \mathcal{T}^{D_1}_t ( s,a) )}{|D_2|} \right)^2 \right] \right)^{1/2} \label{eq:cond-var}
\end{align}
Observe that each trajectory $\textsf{tr} \in D_2$ is generated by independently rolling out $\pi^*$. Therefore, $\frac{1}{|D_2|} \sum_{\textsf{tr} \in D_2} \mathbbm{1} ( \textsf{tr} \in \mathcal{T}^{D_1}_t ( s,a) )$ is an unbiased estimate of $\mathrm{Pr}_{\pi^*} [ \mathcal{T}^{D_1}_t ( s,a) ]$. Therefore the expectation term in \cref{eq:cond-var} is nothing but the variance: letting $\textsf{tr}_1$ be an arbitrary trajectory in $D_2$,
\begin{align}
    &\sum_{s \in \mathcal{S}} \sum_{a \in \mathcal{A}} \sum_{t=1}^H \mathbb{E} \left[ \left| \mathrm{Pr}_{\pi^*} \Big[ \mathcal{T}^{D_1}_t ( s,a) \Big] - \frac{\sum_{\textsf{tr} \in D_2} \mathbbm{1} ( \textsf{tr} \in \mathcal{T}^{D_1}_t ( s,a) )}{|D_2|} \right| \right] \nonumber \\
    &\quad \le \sum_{s \in \mathcal{S}} \sum_{a \in \mathcal{A}} \sum_{t=1}^H \left( \frac{1}{|D_2|} \mathrm{Var} \left[ \mathbbm{1} ( \textsf{tr}_1 \in \mathcal{T}^{D_1}_t ( s,a ) \right] \right)^{1/2} \\
    &\quad \le \sum_{s \in \mathcal{S}} \sum_{a \in \mathcal{A}} \sum_{t=1}^H \left( \frac{1}{|D_2|} \mathrm{Pr}_{\pi^*} \left[ \mathcal{T}^{D_1}_t ( s,a ) \right] \right)^{1/2}
\end{align}
where the last inequality uses the fact that the variance of an indicator function is at most its mean, and that each $\textsf{tr} \in D_2$ is independently drawn by rolling out $\pi^*$. Now, taking expectation with respect to the dataset $D_1$, and by another application of Jensen's inequality,
\begin{align}
    &\sum_{s \in \mathcal{S}} \sum_{a \in \mathcal{A}} \sum_{t=1}^H \mathbb{E} \left[ \left| \mathrm{Pr}_{\pi^*} \Big[ \mathcal{T}^{D_1}_t ( s,a) \Big] - \frac{\sum_{\textsf{tr} \in D_2} \mathbbm{1} ( \textsf{tr} \in \mathcal{T}^{D_1}_t ( s,a) )}{|D_2|} \right| \right] \nonumber \\
    &\quad \le \sum_{s \in \mathcal{S}} \sum_{a \in \mathcal{A}} \sum_{t=1}^H \frac{1}{|D_2|^{1/2}} \left( \mathbb{E} \left[ \mathrm{Pr}_{\pi^*} \left[ \mathcal{T}^{D_1}_t ( s,a ) \right] \right] \right)^{1/2} \\
    &\quad = \sum_{s \in \mathcal{S}} \sum_{t=1}^H \frac{1}{|D_2|^{1/2}} \left( \mathbb{E} \left[ \mathrm{Pr}_{\pi^*} \Big[ \mathcal{E}_{D_1}^{\le t}, s_t = s, a_t = \pi_t^* (s_t) \Big] \right] \right)^{1/2},
\end{align}
where in the last equation, we use the definition of $\mathcal{T}_t^{D_1} (\cdot,\cdot)$. By an application of the Cauchy Schwarz inequality,
\begin{align}
    &\sum_{s \in \mathcal{S}} \sum_{a \in \mathcal{A}} \sum_{t=1}^H \mathbb{E} \left[ \left| \mathrm{Pr}_{\pi^*} \Big[ \mathcal{T}^{D_1}_t ( s,a) \Big] - \frac{\sum_{\textsf{tr} \in D_2} \mathbbm{1} ( \textsf{tr} \in \mathcal{T}^{D_1}_t ( s,a) )}{|D_2|} \right| \right] \nonumber \\
    &\quad\le \sum_{t=1}^H \frac{|\mathcal{S}|^{1/2}}{|D_2|^{1/2}} \left( \sum_{s \in \mathcal{S}} \mathbb{E} \left[ \mathrm{Pr}_{\pi^*} \Big[ \mathcal{E}_{D_1}^{\le t}, s_t = s, a_t = \pi^*_t (s) \Big] \right] \right)^{1/2} \\
    &\quad\le \sum_{t=1}^H \frac{|\mathcal{S}|^{1/2}}{|D_2|^{1/2}} \left( \mathbb{E} \left[ \mathrm{Pr}_{\pi^*} \Big[ \mathcal{E}_{D_1}^{\le t} \Big] \right]\right)^{1/2}.
\end{align}
Therefore, to prove the result it suffices to bound $\mathbb{E} \left[ \mathrm{Pr}_{\pi^*} \left[ \mathcal{E}_{D_1}^{\le t} \right] \right]$, which we carry out in \Cref{lemma:error-prob-small}. Here we show that it is upper bounded by $\lesssim 1 \wedge |\mathcal{S}|H/|D_1|$. Subsequently using $|D_1| = |D_2| = N/2$ completes the proof.

\begin{lemma} \label[lemma]{lemma:error-prob-small}
For any $t \in [H]$, the probability of failure under the expert's policy is upper bounded by,
\begin{equation}
    \mathbb{E} \left[ \mathrm{Pr}_{\pi^*} \Big[\mathcal{E}_{D_1}^{\le t} \Big] \right] \le \frac{4}{9}\frac{|\mathcal{S}|H}{|D_1|}
\end{equation}
\end{lemma}
\begin{proof}
Conditioned on $D_1$, we decompose based on the first failure time (i.e. the first time the event $\mathcal{E}_{D_1}^{\le t}$ is satisfied),
\begin{align}
    \mathrm{Pr}_{\pi^*} \Big[\mathcal{E}_{D_1}^{\le t} \Big| D_1 \Big] &= \mathrm{Pr}_{\pi^*} \Big[\exists \tau \le t : s_\tau \not\in \mathcal{S}_\tau (D_1) \Big| D_1 \Big],\\
    &= \sum\nolimits_{\tau = 1}^t \mathrm{Pr}_{\pi^*} \Big[ \forall \tau' < \tau,\ s_{\tau'} \in \mathcal{S}_{\tau'} (D_1) , s_\tau \not\in \mathcal{S}_\tau (D_1) \Big| D_1 \Big] \\
    &\le \sum\nolimits_{\tau = 1}^t \mathrm{Pr}_{\pi^*} \Big[ s_\tau \not\in \mathcal{S}_\tau (D_1) \Big| D_1 \Big] \\
    &= \sum\nolimits_{\tau = 1}^t \sum\nolimits_{s \in \mathcal{S}} \mathrm{Pr}_{\pi^*} [ s_\tau = s ] \mathbbm{1} (s \not\in \mathcal{S}_\tau (D_1)) \\
    &\le \sum\nolimits_{\tau = 1}^H \sum\nolimits_{s \in \mathcal{S}} \mathrm{Pr}_{\pi^*} [ s_\tau = s ] \mathbbm{1} (s \not\in \mathcal{S}_\tau (D_1)) \label{eq:error-prob-bound}
\end{align}
Taking expectation with respect to the expert dataset,
\begin{equation}
    \mathbb{E} \left[ \mathrm{Pr}_{\pi^*} \Big[\mathcal{E}_{D_1} \Big| D_1 \Big] \right] \le \sum\nolimits_{\tau = 1}^H \sum\nolimits_{s \in \mathcal{S}} \mathrm{Pr}_{\pi^*} [ s_\tau = s ] \mathrm{Pr} [s \not\in \mathcal{S}_\tau (D_1)]
\end{equation}
The proof of the claim immediately follows by invoking \Cref{lemma:BC-prob-small}.
\end{proof}

\subsubsection{Proof of \Cref{lemma:NsimInf:highprob}}

Starting from the bound in \Cref{lemma:reduc} and using the fact that at each state $s$ the expert plays a fixed action $\pi^*_t (s)$ at time $t$,
\begin{align}
    J(\pi^*) - J (\widehat{\pi}) \le 2\sum_{s \in \mathcal{S}} \sum_{t=1}^H \left| \mathrm{Pr}_{\pi^*} \Big[ \mathcal{T}^{D_1}_t ( s,\pi^*_t (s)) \Big] - \frac{\sum_{\textsf{tr} \in D_2} \mathbbm{1} ( \textsf{tr} \in \mathcal{T}^{D_1}_t ( s,\pi^*_t (s)) )}{|D_2|} \right|
\end{align}
Observe that $\mathbbm{1} ( \textsf{tr} \in \mathcal{T}^{D_1}_t ( s,\pi^*_t (s)) )$ is a sub-Gaussian random variable with variance bounded by its expectation. Therefore, by sub-Gaussian concentration \cite{Boucheron2013ConcentrationI}, for each $s \in \mathcal{S}$ and $t \in [H]$, conditioned on $D_1$, with probability $\ge 1 - \frac{\delta}{2 |\mathcal{S}| H}$, 
\begin{align}
    &\left| \frac{\sum_{\textsf{tr} \in D_2} \mathbbm{1} ( \textsf{tr} \in \mathcal{T}^{D_1}_t ( s,\pi^*_t (s)) )}{|D_2|} - \mathrm{Pr}_{\pi^*} \Big[ \mathcal{T}^{D_1}_t ( s,\pi^*_t (s)) \Big] \right| \nonumber\\
    &\le \left( \mathrm{Pr}_{\pi^*} \Big[ \mathcal{T}^{D_1}_t ( s,\pi^*_t (s)) \Big] \right)^{1/2} \sqrt{\frac{2\log (2|\mathcal{S}| H/\delta)}{|D_2|}}
\end{align}
By union bounding over $s \in \mathcal{S}$ and $t \in [H]$, conditioned on $D_1$ with probability $\ge 1 - \frac{\delta}{2}$,
\begin{align}
    &\sum_{t=1}^H \sum_{s \in \mathcal{S}} \left| \frac{\sum_{\textsf{tr} \in D_2} \mathbbm{1} ( \textsf{tr} \in \mathcal{T}^{D_1}_t ( s,\pi^*_t (s)) )}{|D_2|} - \mathrm{Pr}_{\pi^*} \Big[ \mathcal{T}^{D_1}_t ( s,\pi^*_t (s)) \Big] \right| \nonumber\\
    &\le \sum_{t=1}^H \left( \sum_{s \in \mathcal{S}} \mathrm{Pr}_{\pi^*} \Big[ \mathcal{T}^{D_1}_t ( s,\pi^*_t (s)) \Big] \right)^{1/2} \sqrt{\frac{2\log (2|\mathcal{S}| H/\delta)}{|D_2|}} \\
    &\le H |\mathcal{S}|^{1/2} \left( \mathrm{Pr}_{\pi^*} \Big[ \mathcal{E}_{D_1} \Big] \right)^{1/2} \sqrt{\frac{2\log (2|\mathcal{S}| H/\delta)}{|D_2|}} \label{eq:error-prob-NsimInf-1}
\end{align}
Applying \Cref{lemma:error-prob-conc}, with probability $\ge 1 - \delta/2$,
\begin{equation} \label{eq:error-prob-NsimInf-2}
    \mathrm{Pr}_{\pi^*} [\mathcal{E}_{D_1}] \le \frac{4 |\mathcal{S}| H}{9 |D_1|} +
    \frac{3H\sqrt{|\mathcal{S}|} \log (2H/\delta)}{|D_1|}.
\end{equation}
Therefore union bounding the events of \cref{eq:error-prob-NsimInf-1,eq:error-prob-NsimInf-2}, with probability $\ge 1 - \delta$,
\begin{align}
    &\sum_{t=1}^H \sum_{s \in \mathcal{S}} \left| \frac{\sum_{\textsf{tr} \in D_2} \mathbbm{1} ( \textsf{tr} \in \mathcal{T}^{D_1}_t ( s,\pi^*_t (s)) )}{|D_2|} - \mathrm{Pr}_{\pi^*} \Big[ \mathcal{T}^{D_1}_t ( s,\pi^*_t (s)) \Big] \right| \nonumber\\
    &\le H |\mathcal{S}|^{1/2} \left( \frac{4 |\mathcal{S}| H}{9 |D_1|} + \frac{3H\sqrt{|\mathcal{S}|} \log (2H/\delta)}{|D_1|} \right)^{1/2} \sqrt{\frac{2\log (2|\mathcal{S}| H/\delta)}{|D_2|}} \\
    &\lesssim \frac{|\mathcal{S}| H^{3/2}}{N} \left( 1 + \frac{3\log (2|\mathcal{S}| H/\delta)}{\sqrt{|\mathcal{S}|}} \right)^{1/2} \sqrt{\log (2|\mathcal{S}| H/\delta)}.
\end{align}

\subsection{Lower bound in the no-interaction / active settings}

\subsubsection{Proof of \Cref{lemma:cond-is-mimic}}

Fix some policy $\pi \in \Pi_{\mathrm{det}}$. Consider any time $t \in [H]$ and state $s \in \mathcal{S}_t (D)$ which is visited in some trajectory in the dataset at time $t$. If $\pi_t (s)$ does not match the action $\pi^A_t (s)$ revealed by actively querying the expert in a trajectory in $D$ that visits $s$ at time $t$, the likelihood of $\pi$ given $D$ is exactly $0$ (since the expert is deterministic). On the other hand, the conditional probability of observing $(D,A)$ does not depend on the expert's action on the states that were not observed in $D$, since no trajectory visits these states. Since on these states the expert's action marginally follows the uniform distribution over $\mathcal{A}$, the result immediately follows.

\subsubsection{Proof of \Cref{lemma:hatvalue-UB}}

In order to prove this result, define the auxiliary random time $\tau_b$ to be the first time the learner first encounters the state $b$ while rolling out a trajectory. If no such state is encountered, $\tau$ is defined as $H+1$. Formally,
\begin{equation*}
    \tau_b = \begin{cases} \inf \{ t : s_t = b \} &\exists t : s_t = b \\
    H+1 &\text{otherwise}.
    \end{cases}
\end{equation*}
Conditioning on the learner's dataset $(D,A)$, first observe that
\begin{align}
    H - \mathbb{E}_{(\pi^*, \mathcal{M}) \sim \mathcal{P} (D,A)} \left[  J (\widehat{\pi})\right] &= H - \mathbb{E}_{(\pi^*, \mathcal{M}) \sim \mathcal{P} (D,A)} \left[ \mathbb{E}_{\widehat{\pi}} \left[ \sum\nolimits_{t=1}^H \mathbf{r}_t (s_t,a_t) \right]\right] \\
    &\ge \mathbb{E}_{(\pi^*, \mathcal{M}) \sim \mathcal{P} (D,A)} \left[ \mathbb{E}_{\widehat{\pi}} \left[ H - \tau_b + 1 \right]\right] \label{eq:regretbound}
\end{align}
where the last inequality follows from the fact that $\mathbf{r}$ is bounded in $[0,1]$, and the state $b$ is absorbing and offers $0$ reward irrespective of the choice of action. Fixing the dataset $(D,A)$ and the expert's policy $\pi^*$ (which determines the MDP $\mathcal{M} [\pi^*]$), we study $\mathbb{E}_{\widehat{\pi} (D,A)} \left[ H - \tau_b + 1 \right]$ and try to relate it to $\mathbb{E}_{\widehat{\pi} (D,A)} \left[ H - \tau \right]$.

To this end, first observe that for any $t \le H-1$ and state $s \in \mathcal{S}$,
\begin{align}
    \mathrm{Pr}_{\widehat{\pi}} \left[ \tau_b = t + 1, \tau = t, s_t = s\right] &= \mathrm{Pr}_{\widehat{\pi}} \left[ \tau_b = t + 1 | \tau = t, s_t = s \right] \mathrm{Pr}_{\widehat{\pi}} \left[ \tau = t, s_t = s \right]\\
    &= \Big( 1 - \widehat{\pi}_t (\pi^*_t (s) | s) \Big) \mathrm{Pr}_{\widehat{\pi}} \left[ \tau = t, s_t = s \right].
\end{align}
where in the last equation, we use the fact that the learner must play an action other than $\pi^*_t (s_t)$ to visit $b$ at time $t+1$. Next we take expectation with respect to the randomness of $\pi^*$ which conditioned on $(D,A)$ is drawn from $\mathrm{Unif} (\Pi_{\mathrm{mimic}} (D,A))$ which also specifies the underlying MDP $\mathcal{M} [\pi^*]$. Observe that the dependence of the second term $\mathrm{Pr}_{\widehat{\pi}} \left[ \tau = t, s_t = s \right]$ on $\pi^*$ comes from the probability computed with the underlying MDP chosen as $\mathcal{M} [\pi^*]$. However observe that it only depends on the characteristics of $\mathcal{M} [\pi^*]$ till time $t-1$ which are determined by $\pi^*_1,\cdots,\pi^*_{t-1}$. On the other hand, the first term $\left( 1 - \widehat{\pi}_t (\pi^*_t (s) | s) \right)$ depends only on $\pi^*_t$. As a consequence the two terms depend on a disjoint set of random variables, which are independent (since conditionally $\pi^* \sim \Pi_{\mathrm{mimic}} (D,A)$ defined in \cref{eq:Pi.mimic.DA})

Therefore taking expectation with respect to the randomness of $\pi^* \sim \mathrm{Unif} (\Pi_{\mathrm{mimic}} (D,A))$ and $\mathcal{M} = \mathcal{M} [\pi^*]$ (which defines the joint distribution $\mathcal{P} (D,A)$ in \cref{eq:Pi.mimic.DA}),
\begin{align}
    &\mathbb{E}_{(\pi^*, \mathcal{M}) \sim \mathcal{P} (D,A)} \Big[ \mathrm{Pr}_{\widehat{\pi}(D,A)} \left[ \tau_b = t + 1, \tau = t, s_t = s \right] \Big] \nonumber\\
    &= \mathbb{E}_{(\pi^*, \mathcal{M}) \sim \mathcal{P} (D,A)} \Big[ 1 - \widehat{\pi}_t (\pi^*_t (s_t) | s_t) \Big] \ \mathbb{E}_{(\pi^*, \mathcal{M}) \sim \mathcal{P} (D,A)} \Big[ \mathrm{Pr}_{\widehat{\pi}} \left[ \tau = t, s_t = s \right] \Big] \\
    &\overset{(a)}{=} \left( 1 - \frac{1}{|\mathcal{A}|} \right) \mathbb{E}_{(\pi^*, \mathcal{M}) \sim \mathcal{P} (D,A)} \Big[ \mathrm{Pr}_{\widehat{\pi}} \left[ \tau = t, s_t = s \right] \Big]
\end{align}
where in $(a)$, conditioned on $(D,A)$ we use the fact that either $(i)$ $s = b$, in which case $\tau \ne t$ and both sides are $0$, or (ii) if $s \ne b$, then $\tau = t$ implies that the state $s$ visited at time $t$ must not be observed in $D$, so $\pi^*_t (s) \sim \mathrm{Unif} (\mathcal{A})$. Using the fact that $\mathrm{Pr}_{\widehat{\pi}} \left[ \tau_b = t + 1, \tau = t, s_t = s\right] \le \mathrm{Pr}_{\widehat{\pi}} \left[ \tau_b = t + 1 , s_t = s\right]$ and summing over $s \in \mathcal{S}$ results in the inequality,
\begin{equation}
    \mathbb{E}_{(\pi^*, \mathcal{M}) \sim \mathcal{P} (D,A)} \Big[ \mathrm{Pr}_{\widehat{\pi}} \left[ \tau_b = t + 1 \right] \Big] \ge \left( 1 - \frac{1}{|\mathcal{A}|} \right) \mathbb{E}_{(\pi^*, \mathcal{M}) \sim \mathcal{P} (D,A)} \Big[ \mathrm{Pr}_{\widehat{\pi}} \left[ \tau = t \right] \Big]
\end{equation}
Multiplying both sides by $H-t$ and summing over $t=1,\cdots,H$,
\begin{equation}
    \mathbb{E}_{(\pi^*, \mathcal{M}) \sim \mathcal{P} (D,A)} \Big[ \mathbb{E}_{\widehat{\pi}} \left[ H-\tau_b + 1 \right] \Big] \ge \left( 1 - \frac{1}{|\mathcal{A}|} \right) \mathbb{E}_{(\pi^*, \mathcal{M}) \sim \mathcal{P} (D,A)} \Big[ \mathbb{E}_{\widehat{\pi}} \left[ H - \tau \right] \Big]
\end{equation}
here we use the fact that the initial distribution $\rho$ places no mass on the bad state $b$. Therefore, $\mathrm{Pr}_{\widehat{\pi}(D)} \left[ \tau_b = 1 \right] = \rho (b) = 0$. This equation in conjunction with \cref{eq:regretbound} completes the proof.

\subsubsection{Proof of \Cref{lemma:failuretime-sumbound}}

Firstly, in \Cref{lemma:inv-gamma} we show that $\mathbb{E} \left[ \mathrm{Pr}_{\widehat{\pi} (D,A)} [\tau \le \lfloor H/2 \rfloor] \right] \ge 1 - \left( 1 - \gamma \right)^{\lfloor H/2 \rfloor}$ where $\gamma$ is defined as $\sum_{s \in \mathcal{S}} \rho (s) (1 - \rho (s))^N$. Subsequently, in \Cref{lemma:gammabound} we show that $\gamma \gtrsim |\mathcal{S}|/N$. Putting these two results together proves the statement of \Cref{lemma:failuretime-sumbound}.

Along the way to proving \Cref{lemma:inv-gamma}, we introduce an auxiliary result.

\begin{lemma} \label[lemma]{lemma:failuretime-decomp}
Fix the dataset $(D,A)$ collected by the learner, and any policy $\pi^* \in \Pi_{\mathrm{mimic}} (D,A)$ (defined in \cref{eq:Pi.mimic.DA}). Recall that $\tau$ as defined in \Cref{lemma:hatvalue-UB} is the first time $t$ that the learner encounters a state $s_t \ne b$ that has not been visited in $D$ at time $t$.

For some $t \in [H]$, consider $\mathrm{Pr}_{\widehat{\pi}(D)} \left[ \tau = t \right]$ computed with the underlying MDP as $\mathcal{M} [\pi^*]$. Then,
\begin{equation}
    \mathrm{Pr}_{\widehat{\pi} (D,A)} [\tau = t] = \left( 1 - \rho \left( \mathcal{S}_t (D) \setminus \{ b \} \right) \right) \prod\nolimits_{t'=1}^{t-1} \rho \Big( \mathcal{S}_{t'} (D) \setminus \{ b \} \Big) 
\end{equation}
\end{lemma}
\begin{proof}
First observe that, the event $\{ \tau = t \}$ implies that the learner only visits states in $\mathcal{S}_{t'} (D) \cup \{ b \}$ till time $t' < t$, and visits a state in $\mathcal{S}_\tau (D) \cup \{ b \}$ at time $t$. That is,
\begin{align}
    \mathrm{Pr}_{\widehat{\pi}} [\tau = t] &= \mathrm{Pr}_{\widehat{\pi}} \Big[ s_t \not\in \mathcal{S}_t (D) \cup \{ b \}, \ \forall t' < t, s_{t'} \in \mathcal{S}_{t'} (D) \cup \{ b \} \Big] \\
    &= \mathrm{Pr}_{\widehat{\pi}} \Big[ s_t \not\in \mathcal{S}_t (D) \cup \{ b \}, \ \forall t' < t, s_{t'} \in \mathcal{S}_{t'} (D) \setminus \{ b \} \Big]
\end{align}
where in the last equation, we use the fact that by construction of $\mathcal{M} [\pi^*]$, the learner is forced to visit the state $b$ at time $t$ if the state $b$ is visited at any time $t' < t$.

Moreover, since the learner never visits $b$ till time $t-1$, this implies that the learner must play the expert's action at each visited state until time $t-1$ (otherwise the state $b$ is visited with probability $1$ at time $t$). Therefore,
\begin{equation}
    \mathrm{Pr}_{\widehat{\pi}} [\tau = t] = \mathrm{Pr}_{\pi^*} \Big[ s_t \not\in \mathcal{S}_t (D) \cup \{ b \}, \ \forall t' < t, s_{t'} \in \mathcal{S}_{t'} (D) \setminus \{ b \} \Big].
\end{equation}
Since under the policy $\pi^*$ rolled out on $\mathcal{M} [\pi^*]$, the distribution over states induced is i.i.d. across time and drawn from $\rho$, we have that,
\begin{equation}
    \mathrm{Pr}_{\widehat{\pi}} [\tau = t] = \left( 1 - \rho \left( \mathcal{S}_t (D) \cup \{ b \} \right) \right) \prod\nolimits_{t'=1}^{t-1} \rho (\mathcal{S}_{t'} (D) \setminus \{ b \}) 
\end{equation}
However the distribution $\rho$ has no mass on the state $b$. Therefore $\rho \left( \mathcal{S}_t (D) \cup \{ b \} \right) = \rho \left( \mathcal{S}_t (D) \setminus \{ b \} \right)$ and the proof concludes.
\end{proof}

\begin{corollary} \label[corollary]{corr:H/2bound}
$\mathrm{Pr}_{\widehat{\pi} (D,A)} [\tau \le \lfloor H/2 \rfloor] = 1 - \prod\nolimits_{t=1}^{\lfloor H/2 \rfloor} \rho \Big( \mathcal{S}_t (D) \setminus \{ b \} \Big)$.
\end{corollary}

\begin{lemma} \label[lemma]{lemma:inv-gamma}
Fix some policy $\pi^* \in \Pi_{\mathrm{mimic}} (D,A)$ and the MDP as $\mathcal{M} [\pi^*]$. Then,
\begin{equation}
    \mathbb{E} \left[ \mathrm{Pr}_{\widehat{\pi} (D,A)} [\tau \le \lfloor H/2 \rfloor] \right] \ge 1 - \left( 1 - \gamma \right)^{\lfloor H/2 \rfloor}
\end{equation}
where $\gamma = \sum_{s \in \mathcal{S}} \rho (s) (1 - \rho (S))^N$.
\end{lemma}
\begin{proof}
Recall that the learner rolls out policies $\pi_1,\cdots,\pi_N$ to generate trajectories $\textsf{tr}_1,\cdots,\textsf{tr}_N$. First observe that, conditioned on the learner's dataset truncated till the states visited at time $t$,
\begin{align}
    &\mathbb{E} \left[ \prod\nolimits_{t=1}^\tau \rho \Big( \mathcal{S}_t (D) \setminus \{ b \} \Big) \right] - \mathbb{E} \left[ \prod\nolimits_{t=1}^{\tau+1} \rho \Big( \mathcal{S}_t (D) \setminus \{ b \} \Big) \right] \nonumber\\
    &= \mathbb{E} \left[ \prod\nolimits_{t=1}^\tau \rho \Big( \mathcal{S}_t (D) \setminus \{ b \} \Big) \Big(1 - \mathbb{E} \left[ \rho \left( \mathcal{S}_{\tau+1} (D) \setminus \{ b \} \right) \Big| D_{\le \tau, < \tau} \right] \Big) \right] \label{eq:bound-prev}
\end{align}
where in the last equation we use the fact $\mathcal{S}_t (D)$ for all $t \le \tau$ is a measurable function of $D_{\le \tau, < \tau}$. Conditioned on $D_{\le \tau,< \tau}$, consider the distribution over actions $a_\tau^n$ played by the learner in different trajectories. If $a_\tau^n = \pi^*_t (s_\tau^n)$, the state $s_{\tau+1}^n$ is renewed in the distribution $\rho$. If $a_\tau^n$ is any other action, $s_{\tau+1}^n = b$ with probability $1$, and does not provide any contribution to $\rho \left( \mathcal{S}_{\tau+1} (D) \setminus \{ b \} \right)$. Let the random variable $N'$ denote the number of trajectories that have already visited $b$ prior to time $\tau$ or play an action other than the expert's action at time $\tau$. By linearity of expectation,
\begin{align}
    1 - \mathbb{E} \left[ \rho \left( \mathcal{S}_{\tau+1} (D) \setminus \{ b \} \right) \Big| D_{\le \tau, < \tau} \right] &= \mathbb{E} \left[ \sum\nolimits_{s \in \mathcal{S} \setminus \{ b \}} \rho (s) \left(1 - \rho (s) \right)^{N'} \middle| D_{\le \tau, < \tau} \right] \\
    &\ge \sum\nolimits_{s \in \mathcal{S} \setminus \{ b \}} \rho (s) \left(1 - \rho (s) \right)^N \label{eq:bound-next}
\end{align}
Recalling that $\gamma$ is defined as the constant $\sum_{s \in \mathcal{S}} \rho (s) \left(1 - \rho (s) \right)^N$ and $\rho (b) = 0$, from \cref{eq:bound-prev,eq:bound-next},
\begin{equation}
    \mathbb{E} \left[ \prod\nolimits_{t=1}^{\tau+1} \rho \Big( \mathcal{S}_t (D) \setminus \{ b \} \Big) \right] \ge (1 - \gamma) \mathbb{E} \left[ \prod\nolimits_{t=1}^\tau \rho \Big( \mathcal{S}_t (D) \setminus \{ b \} \Big) \right] \label{eq:rec}
\end{equation}
We also have that $\mathbb{E} [\rho (\mathcal{S}_1 (D) \setminus \{ b \})] = 1 - \sum_{s \in \mathcal{S} \setminus \{ b \}} \rho (s) (1 - \rho (s))^N = 1 - \gamma$ since the initial state $s$ in each trajectory in $D$ is sampled independently and identically from $\rho$. Using this fact and recursing \cref{eq:rec} over $\tau = 1,\cdots,\lfloor H/2 \rfloor-1$ gives,
\begin{equation}
    \mathbb{E} \left[ \prod\nolimits_{t=1}^{\lfloor H/2 \rfloor} \rho \Big( \mathcal{S}_t (D) \setminus \{ b \} \Big) \right] \ge (1 - \gamma)^{\lfloor H/2 \rfloor}.
\end{equation}
Invoking \Cref{corr:H/2bound} completes the proof.
\end{proof}

\begin{lemma} \label[lemma]{lemma:gammabound}
$\gamma$, defined in \Cref{lemma:inv-gamma} as $\sum_{s \in \mathcal{S}} \rho (s) (1 - \rho (s))^N$ is $\ge \frac{|\mathcal{S}|-2}{e (N+1)}$.
\end{lemma}
\begin{proof}
By the definition of $\rho$, we have that,
\begin{equation}
    \gamma = \sum_{s \in \mathcal{S}} \rho (s) (1 - \rho (s))^N \overset{(i)}{\ge} \frac{|\mathcal{S}|-2}{N+1} \left( 1 - \frac{1}{N+1} \right)^N \ge \frac{|\mathcal{S}|-2}{e(N+1)}.
\end{equation}
where in $(i)$ we lower bound by only considering the $|\mathcal{S}|-2$ states having mass $=\frac{1}{N+1}$ under $\rho$.
\end{proof}

\subsection{Lower bound in the known-transition setting}

\subsubsection{Proof of \Cref{lemma:cond-is-mimic:Inf}}

The proof of this result closely follows that of \Cref{lemma:cond-is-mimic}. Fix some policy $\pi \in \Pi_{\mathrm{det}}$. Consider any time $t \in [H]$ and state $s \in \mathcal{S}_t (D)$ which is visited in some trajectory in the dataset at time $t$. If $\pi_t (s)$ does not match the unique action $a^*_t (s)$ played at time $t$ in any trajectory in $D$ that visits $s$ at this time, the likelihood of $\pi$ given $D$ is exactly $0$ (recall we assume that the expert's policy is deterministic). On the contrary, the conditional probability of observing the expert dataset $D$ does not depend on the expert's action on the states that were not observed in $D$, since no trajectory visits these states. On these states the expert's action marginally follows the uniform distribution over $\mathcal{A}$. Thus the result follows.

\subsubsection{Proof of \Cref{lemma:errorbound:Inf}}

Observe that,
\begin{align}
    &\mathbb{E}_{(\pi^* , \mathbf{r}) \sim \mathcal{P}' (D)} \left[ H - J_{\mathbf{r}} (\widehat{\pi} (D,P,\rho)) \right] \nonumber\\
    &= \mathbb{E}_{(\pi^* , \mathbf{r}) \sim \mathcal{P}' (D)} \left[ \mathbb{E}_{\widehat{\pi}} \left[ \sum\nolimits_{t=1}^H 1 - \mathbf{r}_t (s_t,a_t) \right]\right] \\
    &\ge \sum\nolimits_{t=1}^H \mathbb{E}_{(\pi^* , \mathbf{r}) \sim \mathcal{P}' (D)} \left[ \mathbb{E}_{\widehat{\pi}} \left[ \mathbbm{1} (s_1 \not\in \mathcal{S}_1 (D)) \Big( 1 - \mathbf{r}_t (s_t,a_t) \Big) \right]\right] \label{eq:uncond-bound}
\end{align}
By construction of the $\mathcal{M} [\pi^*]$ and $P$ each state $s \in \mathcal{S}$ is absorbing. Therefore, $s_1 \not\in \mathcal{S}_1 (D) \iff \{ \forall t \in [H], \ s_t \not\in \mathcal{S}_t (D) \}$. By the structure of the reward function $\mathbf{r} [\pi^*]$, the learner accrues a reward of $1$ at some state if and only if the learner plays the expert's action at this state. Therefore, $\mathbf{r}_t (s_t,a_t) = \mathbbm{1} (a_t = \pi^*_t (s_t))$ and,
\begin{align}
    &\mathbb{E}_{(\pi^* , \mathbf{r}) \sim \mathcal{P}' (D)} \left[ \mathbb{E}_{\widehat{\pi}} \left[ \mathbf{r}_t (s_t,a_t) \Big| s_1 \not\in \mathcal{S}_1 (D) \right]\right] \nonumber\\
    &= \mathbb{E}_{(\pi^* , \mathbf{r}) \sim \mathcal{P}' (D)} \left[ \mathbb{E}_{\widehat{\pi}} \left[ \mathbbm{1} (a_t = \pi^*_t (s_t)) \Big| s_1 \not\in \mathcal{S}_1 (D) \right]\right] \label{eq:reward-replace}
\end{align}
From \Cref{lemma:cond-is-mimic:Inf} observe that conditioned on $D$, the expert's policy $\pi^*$ is sampled uniformly from $\Pi_{\mathrm{mimic}} (D)$. Since we condition on $s_1 \not\in \mathcal{S}_1 (D) \iff s_t \not\in \mathcal{S}_t (D)$ the state $s_t$ is not visited in any trajectory in $D$ at time $t$. This implies that the expert's action $\pi^*_t (s_t)$ is uniformly sampled from $\mathcal{A}$. Therefore,
\begin{equation}
    \mathbb{E}_{\widehat{\pi}} \left[ \mathbb{E}_{(\pi^* , \mathbf{r}) \sim \mathcal{P}' (D)} \left[ \mathbbm{1} (a_t {=} \pi^*_t (s_t)) \Big| s_1 \not\in \mathcal{S}_1 (D) \right]\right] = \frac{1}{|\mathcal{A}|} \sum_{a \in \mathcal{A}} \mathbb{E}_{\widehat{\pi}} \left[ \mathbbm{1} (a_t {=} a) \Big| s_1 \not\in \mathcal{S}_1 (D) \right] = \frac{1}{|\mathcal{A}|}. \nonumber
\end{equation}
Plugging this into \cref{eq:reward-replace} and subtracting $1$ from both sides we get that,
\begin{equation}
    \mathbb{E}_{(\pi^* , \mathbf{r}) \sim \mathcal{P}' (D)} \left[ \mathbb{E}_{\widehat{\pi}(D,P,\rho)} \left[ 1 - \mathbf{r}_t (s_t,a_t) \Big| s_1 \not\in \mathcal{S}_1 (D) \right]\right] = 1 - \frac{1}{|\mathcal{A}|}.
\end{equation}
Plugging this back into \cref{eq:uncond-bound} we get that,
\begin{equation}
    \mathbb{E}_{(\pi^* , \mathbf{r}) \sim \mathcal{P}' (D)} \left[ H - J_{\mathbf{r}} (\widehat{\pi} (D,P,\rho)) \right] \ge H \left( 1 - \frac{1}{|\mathcal{A}|} \right) \mathrm{Pr}_{\widehat{\pi} (D, P,\rho)} \left[ s_1 \not\in \mathcal{S}_1 (D) \right]
\end{equation}
Since $s_1$ is sampled independently from $\rho$, the proof of the result concludes.

\subsubsection{Proof of \Cref{lemma:NsimInf-lasting}}

Note that the dataset $D$ follows the posterior distribution generated by rolling out $\pi^*$ for $N$ episodes when $\pi^*$ is drawn from the uniform prior $\mathrm{Unif} (\Pi_{\mathrm{det}})$. Irrespective of the choice of $\pi^*$, note that the initial distribution over states is still $\rho$. Therefore,
\begin{align}
    \mathbb{E} [1 - \rho (\mathcal{S}_1 (D)] &= \sum_{s \in \mathcal{S}} \rho (s) (1 - \rho (s))^{N} \\
    &\overset{(i)}{\ge} \frac{|\mathcal{S}|-1}{N+1} \left( 1 - \frac{1}{N+1} \right)^{N} \ge \frac{|\mathcal{S}|-1}{e(N+1)}
\end{align}
where in $(i)$ we lower bound by considering only the $|\mathcal{S}|-1$ states having mass $\frac{1}{N+1}$ under $\rho$. Plugging this back into \cref{eq:Rlowerbound-3:Inf} completes the proof of the theorem.

\end{appendices}

\end{document}